\documentclass[12pt]{article}
\usepackage{amsmath,amssymb,amsfonts, amsbsy, epsfig, subfig}
\usepackage[usenames]{color}
\usepackage{rotating}

\usepackage{graphicx}
\usepackage{color}
\usepackage{booktabs}
\usepackage{verbatim}
\usepackage{latexsym}
\usepackage{psfrag}
\usepackage[hyperfootnotes=false]{hyperref}
\usepackage{amsfonts}
\usepackage{amsmath}
\usepackage{amssymb}
\usepackage{bbm}
\usepackage{amsthm}
\usepackage{comment}
\usepackage{enumerate}
\usepackage{url}
\usepackage{graphicx}
\usepackage{rotating}
\usepackage{epsfig}
\usepackage[stable]{footmisc}
\usepackage{endnotes}
\usepackage[english]{babel}
\usepackage{makecell}
\usepackage{rotating}
\usepackage{multirow}
\usepackage[hmargin=3.45cm,vmargin=3cm]{geometry}
\usepackage{algorithm} %format of the algorithm

\usepackage{algorithmic} %format of the algorithm
\theoremstyle{definition}

\theoremstyle{plain}
\newtheorem{theorem}{Theorem}

\newtheorem{lemma}{Lemma}

\theoremstyle{definition}
\newtheorem{definition}{Definition}

\theoremstyle{remark}

\renewcommand{\hat}[1]{\widehat{#1}}

\newcommand{\bX}{{\bf X}}
\newcommand{\bx}{{\bf x}}

\newcommand{\bw}{{\bf w}}
\newcommand{\bz}{{\bf z}}

\newcommand{\mnorm}[1]{\left\vert\kern-1.5pt\left\vert\kern-1.5pt\left\vert #1\right\vert\kern-1.5pt\right\vert\kern-1.5pt\right\vert}
\newcommand{\tc}{\tilde{C}}

\newcommand{\tp}{\tilde{\pi}}

\begin{document}

\title{ Sparse K-Means with $\ell_{\infty}/\ell_0$ Penalty for High-Dimensional Data Clustering\footnote { Xiangyu Chang, Yu Wang and Zongbe Xu are in the Department of Statistics, Xi'an Jiaotong University (Email:xiangyuchang@gmail.com, shif.wang@gmail.com and zbxu@mail.xjtu.edu.cn). Rongjian Li is in the Department of Computer Science, Old Dominion University (Email:rli@cs.odu.edu).}}
\author{Xiangyu Chang, Yu Wang, Rongjian Li and Zongben Xu}

\date{}

\maketitle
\vspace{-1cm}

\begin{abstract}
Sparse clustering, which aims to find a proper partition of an extremely high-dimensional data set with redundant noise features, has been attracted more and more interests in recent years. The existing studies commonly solve the problem in a framework of maximizing the weighted feature contributions subject to a $\ell_2/\ell_1$ penalty. Nevertheless, this framework has two serious drawbacks: One is that the solution of the framework unavoidably involves a considerable portion of redundant noise features in many situations, and the other is that the framework neither offers intuitive explanations on why this framework can select relevant features nor leads to any theoretical guarantee for feature selection consistency.

In this article, we attempt to overcome those drawbacks through developing a new sparse clustering framework which uses a $\ell_{\infty}/\ell_0$ penalty.
First, we introduce new concepts on optimal partitions and noise features for the high-dimensional data clustering problems, based on which the previously known framework can be intuitively explained in principle. Then, we apply the suggested $\ell_{\infty}/\ell_0$ framework to formulate a new sparse k-means model with the $\ell_{\infty}/\ell_0$ penalty ($\ell_0$-k-means for short). We propose an efficient iterative algorithm for solving the $\ell_0$-k-means. To deeply understand the behavior of $\ell_0$-k-means, we prove that the solution yielded by the $\ell_0$-k-means algorithm has feature selection consistency whenever the data matrix is generated from a high-dimensional Gaussian mixture model. Finally, we provide experiments with both synthetic data and the Allen Developing Mouse Brain Atlas data to support that the proposed $\ell_0$-k-means exhibits better noise feature detection capacity over the previously known sparse k-means with the $\ell_2/\ell_1$ penalty ($\ell_1$-k-means for short).

\end{abstract}


%\newpage

\section{Introduction}
Clustering is an unsupervised learning technique for discovering hidden group structures in
data sets. It partitions the whole sample set into different groups such that each group has its own unique property. The commonly used algorithms for clustering include k-means clustering~\cite{macqueen1967,steinhaus1956}, hierarchical clustering~\cite{hastie2009elements}, model-based clustering~\cite{dempster1977} and spectral clustering~\cite{von2007}.

With rapid expansion of the capacity in automatic data generation and acquisition, we encountered the clustering problems with a huge number of features. The conventional clustering algorithms treat all features equally, and attribute them with the same weight in a grouping process. In many real situations, however, only a small portion of features is responsible and important for determining the cluster structures. For example, only a small portion of genes is responsible for some biological activities such as protein synthesis. If each gene is viewed as a feature, those genes activated in the biological process are known as {\it relevant features}, and the others can be thought as {\it noise features}. The large size of noise features usually makes traditional clustering methods unpractical and inefficient. Thus, how to identify relevant features and eliminate noise features simultaneously and automatically is of great importance for clustering of high-dimensional data.

There have been several approaches to address the feature selection problem for clustering. One approach is to do dimension reduction before clustering, say, using principle components analysis (PCA)~\cite{chang1983} or nonnegative matrix factorization~\cite{lee1999} (NMF). However, there are evidences showing that principal components do not actually provide reasonable partition of the original data~\cite{chang1983}. Another approach is to perform the penalized model-based clustering. This approach assumes that the data matrix is generated from a known mixture distribution. Then the clusters are detected by fitting data into a log-likelihood function with $\ell_1$ penalty~\cite{raftery2006,wang2008,pan2007}. Nevertheless, the computational hurdle of fitting such model is still formidable, especially when the dimension is very high. Recently, Witten and Tibshirani~\cite{witten2010} suggested a framework of sparse clustering. The framework optimizes a criterion involving feature weights using both a $\ell_1$ penalty and a $\ell _2$ penalty ($\ell_2/\ell_1$ penalty in short). They particularly developed a sparse k-means method for clustering, called $\ell_1$-k-means, which maximizes the weighted Between-Cluster Sum of Squares (BCSS) with the $\ell_2/\ell_1$ penalty, and used Gap Statistics~\cite{tibshirani2001} to select the tunning parameter for determining the number of non-zero weighted features. The $\ell_1$-k-means works well, but still often keeps a considerable portion of redundant features. In other words, many noise features are still not depressed in the final clustering results. This phenomenon has been found in the experimental example they used in their seminal paper~\cite{witten2010}: when 60 observations were generated from 3 clusters involving 50 relevant features and 150 noise features, the $\ell_1$-k-means kept all the noise features in the final clustering result. %We will

It is known that the $\ell_0$ penalty is the most essential sparsity measure, while the $\ell_1$ penalty is only a best convex relaxation of $\ell_0$ penalty. Thus, we naturally expect to use the $\ell_0$ penalty to improve the feature selection performance. However, directly putting the $\ell_0$ penalty into the sparse clustering framework~\cite{witten2010} makes the problem intractable. Even if it is tractable, the solution defined cannot be interpreted. To overcome this difficulty, a new sparse clustering framework using both a $\ell_{\infty}$ and a $\ell_0$ penalty ($\ell_{\infty}/\ell_0$ penalty for short) is proposed. As a realization of this new framework, we develop a $\ell_0$-k-means method for performing clustering. We find that the $\ell_0$-k-means is extremely easy to implement and interpret. What's more, we can show that the $\ell_0$-k-means exhibits much better noise feature detection capacity compared with $\ell_1$-k-means.

In the theoretical point of view, one of the most important issues in high-dimensional data analysis is to analyze the model behavior as the dimension (number of feature or number of variable) grows with the sample size. For some specific supervised problems, there are a huge number of literatures~\cite{zhao2006model,wainwright2009sharp,negahban2012unified,fan2010} in this field. For example,~\cite{zhao2006model,wainwright2009sharp,fan2010} proved the feature (variable) selection consistency property for the penalized regression models and~\cite{negahban2012unified} developed a unified framework for analyzing error bounds of M-estimators with decomposable regularization for high-dimensional data. However, there is a little theoretical analysis for high-dimensional data clustering problems although a few high-dimensional data clustering methods~\cite{pan2007,witten2010} have been reported. Why it is so is that it is really difficult to rigorously define what clusters the clustering methods are looking for and what features are the noise features in the high-dimensional setting (when the dimension is comparable to or larger than the sample size).

In this paper we circumvent this difficulty by redefining the notion on what is an
optimal partition and what are the noise features in high-dimensional statistics setting from a very intuitive way. We then show that the new definitions are well-defined and can be used to interpret the proposed framework and the $\ell_0$-k-means algorithm very explicitly. Based on this, we further prove that the solution of the suggested $\ell_0$-k-means algorithm has a feature selection consistency property.
%To the best of our knowledge, this is the first attempt to understand the optimal partitions and noise features of high-dimensional data clustering in this way.

The rest of the paper is organized as follows. In Section \ref{sparseframework}, we first introduce the necessary notion and notations for the research and then we analyze the classical k-means framework and formalize new definitions of optimal partition and noise features. Based on the new definitions, we present the new sparse clustering framework and the $\ell_0$-k-means. To implement the $\ell_0$-k-means, an efficient iterative algorithm is developed. We further prove that the solution defined by the $\ell_0$-k-means algorithm has a feature selection consistency property under a set of reasonable conditions when the data matrix is generated from a high-dimensional Gaussian mixture model. In Section \ref{simulations}, a series of simulations for synthetic data and Allen Developing Mouse Brain Atlas data set are provided to evaluate the performances of $\ell_0$-k-means, $\ell_1$-k-means and standard k-means. In Section \ref{conclusions}, we conclude the paper with some useful remarks. All the detailed proofs of theoretical results stated in the paper are presented in Section \ref{appendix}, as an appendix of the paper.

\section{Sparse Clustering Framework with $\ell_{\infty}/\ell_0$ Penalty}\label{sparseframework}
\subsection{Notion and Notations}
Let $\bX\in\mathbb{R}^{n\times p}$ be a data set in a matrix form with $n$ observations and $p$ features. We assume that $\bx_i\in\mathbb{R}^p$ and $\bX_j\in\mathbb{R}^n$ are the $i^{th}$ row and $j^{th}$ column of $\bX$ respectively. Thus, $\bX=[\bX_1,\bX_2,\dots,\bX_p]=[\bx_1^\top,\bx_2^\top,\dots,\bx_n^\top]^\top=(x_{ij})_{n\times p}$. It is well known that the standard k-means clustering groups the data by finding a partition $\mathcal{C}=\{{C}_1,{C}_2,\dots,{C}_K\}$ such that the sum of distances between the empirical means ${\mu}=(\mu_1,\mu_2,\dots,\mu_K)^\top$ of the clustering and the corresponding points is minimized. Therefore, it can be formulated as an optimization problem of the form:
\begin{eqnarray}\label{kmean}
\min_{\mathcal{C},\mu}\sum_{k=1}^K\sum_{\bx_i\in C_k}d(\bx_i,\mu_k),
\end{eqnarray}
where $d:\mathbb{R}^p\times\mathbb{R}^p\rightarrow\mathbb{R}$ is the so-called {\it dissimilarity measure} satisfying $d(a,a)=0,d(a,b)\geq 0$ and $d(a,b)=d(b,a)$. The dissimilarity measure between vectors $\bx_i$ and $\bx_j$ is very commonly chosen to be the square of Euclidean distance, that is, $d(\bx_i,\bx_j)=\|\bx_i-\bx_j\|^2=\sum_{l=1}^p(x_{il}-x_{jl})^2$.

More generally, an operational definition of clustering can be stated as follows: Given a representation of $n$ objects, find $K$ groups based on a measure of dissimilarity such that objects within the same group are alike but objects in different groups are disparate~\cite{jain2010}. The k-means naturally satisfies this definition if we rewrite (\ref{kmean}) in a form of {\it Within-Cluster Sum of Square (WCSS)} as follows:
\begin{eqnarray}\label{WCSS}
\sum_{k=1}^K\frac{1}{n_k}\sum_{i,i'\in C_k}\sum_{j=1}^pd_{ii'j},
\end{eqnarray}
where $n_k=|C_k|$ is the cardinality of cluster $C_k$ and $d_{ii'j}=(x_{ij}-x_{i'j})^2$. In practice, it is sometimes more convenient to use {\it Between-Cluster Sum of Squares (BCSS)} defined by
\begin{eqnarray}\label{BCSS}
\sum_{j=1}^p\Big\{\frac{1}{n}\sum_{i=1}^n\sum_{i'=1}^nd_{ii'j}-\sum_{k=1}^K\frac{1}{n_k}\sum_{i,i'\in C_k}d_{ii'j}\Big\}.
\end{eqnarray}
Note that minimizing WCSS amounts to maximizing BCSS, and so the minimization of (\ref{WCSS}) is equivalent to the maximization of (\ref{BCSS}). Furthermore, if we denote
\begin{eqnarray}\label{aj}
a_j\triangleq\frac{1}{n}\sum_{i,i'}^nd_{ii'j}-\sum_{k=1}^K\frac{1}{n_k}\sum_{i,i'\in C_k}d_{ii'j},j=1,\dots,p,
\end{eqnarray}
then $a_j$ is the $j^{th}$ component of BCSS which can be considered as a function
only with respect to the sample values of the $j^{th}$ feature and the partition $\mathcal{C}$. Note that we have abused $a_{j}$ and $a_j(\mathcal{C})$ here for simplicity, but actually calculating $a_j$ depends on $\mathcal{C}$. With the formulation (\ref{aj}), Witten and Tibshirani~\cite{witten2010} generalized the BCSS form (\ref{BCSS}) to an optimization problem of the general form
\begin{eqnarray}\label{framework}
\max_{\Theta(\mathcal{C})\in D}\Big\{\sum_{j=1}^p f_j(\bX_j,\Theta(\mathcal{C}))\Big\},
\end{eqnarray}
where $f_j(\bX_j,\Theta(\mathcal{C}))$ is a function that involves only the $j^{th}$ feature of the data, and $\Theta(\mathcal{C})$ is a parameter restricted to a set $D$. They further defined a {\it sparse clustering framework} as
\begin{eqnarray}\label{l1framework}
&\max\limits_{\bw, \Theta(\mathcal{C})\in D}&\Big\{\sum_{j=1}^p w_jf_j(\bX_j,\Theta(\mathcal{C}))\Big\}\\\nonumber
&\text{s.t.}&\ \|\bw\|_2\leq 1,\ \|\bw\|_1\leq s,\ w_j\geq 0,\ \forall j,
\end{eqnarray}
where $s$ is a tunning parameter, $\|\cdot\|_2$ is the Euclidean norm, $\|\cdot\|_1$ is the $\ell_1$ norm, and $\bw=(w_1,w_2,\dots,w_p)^\top$. Here, $w_j$ can be interpreted as the contribution of the $j^{th}$ feature to the objective function (\ref{l1framework}). A larger value of $w_j$ indicates a feature that contributes more to the clustering results. Moreover, they replaced $f_j(\bX_j,\Theta(\mathcal{C}))$ by the $a_j$, defined as in (\ref{aj}). (\ref{l1framework}) then becomes the following $\ell_1$-k-means model:
\begin{eqnarray}\label{l1kmeans}
&\max\limits_{\mathcal{C}, \bw}&\sum_{j=1}^pw_j\Big(\frac{1}{n}\sum_{i=1}^n\sum_{i'=1}^nd_{ii'j}-\sum_{k=1}^K\frac{1}{n_k}\sum_{i,i'\in C_k}d_{ii'j}\Big)\\\nonumber
&\text{s.t.}&\ \|\bw\|_2\leq 1,\ \|\bw\|_1\leq s,\ w_j\geq 0,\ \forall j.
\end{eqnarray}

The $\ell_1$-k-means works well, but there are two serious drawbacks. One is that the solution of the $\ell_1$-k-means unavoidably involves a considerable portion of redundant noise features in many situations.  In other words, many noise features are still not depressed in the final clustering results. This phenomenon has been found in the Witten and Tibushirani's~\cite{witten2010} experimental study: when 60 observations were generated from 3 clusters involving 50 relevant features and 150 noise features, the $\ell_1$-k-means kept all the noise features in the final clustering result. And the other is that the $\ell_1$-k-means neither offers intuitive explanations on why it can select relevant features nor offers any theoretical guarantee for feature selection consistency.

Therefore, a natural question is: could we have a new sparse k-means framework within which all those drawbacks of $\ell_1$-k-means can be dismissed? To answer this question, we will reformulate the classical k-means to accommodate the high-dimensional data clustering setting based on a new formulation of definition on the optimal partition and noise features. According to the new formulation, we develop a new $\ell_0$-k-means to overcome the $\ell_1$-k-means' drawbacks.

\subsection{Reformulation of K-Means for High-dimensional Data Clustering}\label{reform}
We start with the definition of optimal partition in the classical k-means clustering model and analyze why it is inappropriate for high-dimensional data clustering. Then we propose an alternative and show the appropriateness of the new definition.

In the classical clustering settings~\cite{pollard1981strong}, the k-means is characterized by their centroids $\mu=(\mu_1,\dots,\mu_K)^\top$, and the optimal $\mu^*$ is defined as the minimizer of an {\it expected risk function}
\begin{equation}\label{risk}
\mu^*\in\arg\min_{\mu}\int\min_{i=1,\dots,K}\|\bx-\mu_i\|^2m(d\bx),
\end{equation}
where $\bx\in\mathbb{R}^p$ is sampled from a probability measure $m(\bx)$.
After getting $\mu^*$, the {\it optimal partition} $\mathcal{C}^*$ of the samples is determined by calculating which centroid each sample is closest to. In this formulation, the dimension of samples is fixed.
For high-dimensional statistical problems, however, the dimension of samples is no longer fixed, and the dimensions of different samples may vary. Such variation
plays a significant role in theoretical behavior of high-dimensional statistics problems, because we really care about the relationship between the number of samples and the number of features~\cite{fan2010}.

Nevertheless, as we dig further into this problem, it seems that there is no reasonable way to extend the previous formulation (\ref{risk}) to fit into the high-dimensional statistics setting. The difficulty lies in the implicit relationship between the centroids $\mu^*$ and the probability measure $m$. When dimension $p$ varies, we have to define different probability measures $m_p$ ($m_p$ means a probability measure varying with $p$). This may result in different optimal centroids $\mu^{*}(p)$ leading to a confusion because as the dimension grows, the same sample might be categorized into different clusters and it is lack of an universal optimal partition that can be used to judge whether the estimated partition $\hat{\mathcal{C}}_{n,p}$ is good or not. An apparent way to define such universal optimal partition is to consider the limit of $\mu^*(p)$ along each dimension. But this breaks down when the limit of $\mu^*(p)$ doesn't exist or different centroids tend to the same limit. This does happen even when probability measure $m_p$ is very simple. {For example, when $m_p$ is the uniform distribution over an unit ball of $p$-dimensional space and $K=2$, we can test that the optimal solution is $\mu_1^*(p)=\frac{2}{p+1}\frac{\Gamma(p/2+1)}{\Gamma((p+1)/2)\Gamma(1/2)}(1,0,..,0),\mu_2^*(p)=-\frac{2}{p+1}\frac{\Gamma(p/2+1)}{\Gamma((p+1)/2)\Gamma(1/2)}(1,0,..,0)$ which both tend to $0$ along any dimension} ({This can be verified by solving Eq. (\ref{risk})} and $\Gamma$ is the standard gamma function). This difficulty is deeply rooted in the intriguing relationship between the cluster centroids $\mu$ and the probability measure $m$. In order to escape from this difficulty, we abandon this framework and try to pave a new way to define the optimal partition. This new way should take the variation of dimensions into consideration, and, in particular, the defined optimal partition should be fixed in the situations when dimension grows.

Our new way is motivated by considering a different version of the risk function (\ref{risk}). Suppose that there are $n$ samples and we approximate the probability measure $m$ by the empirical measure $m_n$. Conditioned on a partition $\mathcal{C}$, we then have (omit the constants)
\begin{eqnarray}
\mathbb{E}\{\min_\mu\int\min_{k=1,\dots,K}\|\bx-\mu_k\|^2m_n(d\bx)|\mathcal{C}\}&=&
\mathbb{E}\{\min_\mu\sum_i\min_{k=1,\dots,K}\|\bx_i-\mu_k\|^2|\mathcal{C}\}\\\nonumber
&=&\mathbb{E}\{\sum_k\frac{1}{n_k}\sum_{i,i'\in C_k}\sum_j(x_{ij}-x_{i'j})^2|\mathcal{C}\}\nonumber.
\end{eqnarray}
Note that this last term is the expectation of WCSS, defined in (\ref{WCSS}). This observation prompts us to define the optimal partition in the following way:

\begin{definition}\label{defoptimal}
Given a data matrix $\bX$,
% $\bx_i$ is the $i^{th}$ row of $\bX$ representing a sample.
the {\it optimal partition} $\mathcal{C}^*$ of $\bX$ is the partition that maximizes the expectation of BCSS, i.e.,
\begin{eqnarray}\label{optimialpartion}
\mathcal{C}^*\triangleq \arg\max_\mathcal{C} \sum_{j=1}^p \mathbb{E}[a_j(\mathcal{C})],
\end{eqnarray}
where $a_j$ is defined in Eq. (\ref{aj}).
\end{definition}
From this definition, we can define noise features in a natural way.
\begin{definition}\label{noise}
If the $j^{th}$ feature for any partition $\mathcal{C}$ satisfies
\begin{equation}
\mathbb{E} [a_j(\mathcal{C}^*)]=\mathbb{E} [a_j(\mathcal{C})],
\end{equation}
then this feature is a {\it noise feature}.
\end{definition}

By this definition, a noise feature is a feature that makes all possible partitions attain the same expected BCSS value with respect to this feature. Thus, a noise feature would make no contribution for seeking proper clusters. This is why we call such a feature the noise feature. The features that are not noise features will be called {\it relevant features}.

Theorem \ref{optimial_partition_theory}, whose proof will be given in Appendix, shows that the definition of optimal partition in Definition \ref{defoptimal} is reasonable and, according to Definition \ref{noise}, the noise feature does exist. %Before show that, we need suppose the data matrix is generated form a very general mixture model, detailed as follows.

To state Theorem \ref{optimial_partition_theory}, we need some new notion and notations. We assume that each sample $\bx_i,i=1,2,\dots,n,$ equips with an {\it indicator variable} $\bz_{i}=(z_{i1},z_{i2},\dots,z_{iK})^\top$, where $z_{ik}\in\{0,1\}$ and $\sum_{k=1}^Kz_{ik}=1$. If $z_{ik}=1$ then we say that the $i^{th}$ sample belongs to the $k^{th}$ cluster. Furthermore, suppose that the indicator variables are i.i.d which are drawn from a multinomial distribution, and $\bx_i$ is from a distribution $\mathcal{F}_k(\mu_k,\Sigma_k)$ ($\mu_k$ and $\Sigma_k$ are the mean and the covariance matrix for the $k^{th}$ cluster) when $z_{ik}=1$. Thus, $\bx_i$ satisfies a mixture distribution, that is
\begin{equation}\label{step1}
\mathbb{P}(\bx_i|\bz_i)=\prod_{k=1}^K[\mathcal{F}_k(\mu_k,\Sigma_k)]^{z_{ik}}.
\end{equation}

In order to support the reasonability of Definition \ref{defoptimal} and \ref{noise}, we consider a little more specific setting. Suppose that each element $x_{ij}$ is uncorrelated to each other for all $i=1,2,\dots,n,$ and $ j=1,2,\dots,p$, and each $x_{ij}$ obeys $\mathcal{F}(\mu_{ij},1)$ where $\mu_{ij}$ is defined as:
\begin{equation}
\mu_{ij} = \left\{\begin{array}{cc}
  \mu_k & \text{if}\ i\in C_k,j\leq p^* \\
  0& \text{if}\ \forall i, p^*<j\leq p
\end{array}\right.,\label{step2}
\end{equation}
and $\mu_k$ are all constants with $\mu_k\neq\mu_l$ when $k\neq l$.
Thus, we have a {\it natural partition} $\mathcal{C}^*=\{C^*_1,\cdots,C^*_K\}$ based on this setting. Assume that we have an estimated partition ${\tilde{\mathcal{C}}}=\{\tilde{C}_1,\tilde{C}_2,\dots,\tilde{C}_K\}$. Denote $\boldsymbol{\pi}=(\pi_{kk'})\in\mathbb{R}^{K\times K}$ with $\pi_{kk'}$ being the proportion of samples in both ${C}^*_{k}$ and $\tilde{C}_{k'}$. Consequently, $\sum_{k,k'}\pi_{kk'}=1$. For future use, we define also $\tilde{\pi}_{k'}=\sum_{k}\pi_{kk'}$ and $\pi_{k}=\sum_{k'}\pi_{kk'}$. Based on this formulation, we define the Error Clustering Rate (ECR) of ${\mathcal{C}}$ to be one minus its {\it purity}: $ECR({\mathcal{C}})=1-purity({\mathcal{C}})$, where $purity({\mathcal{C}})=\sum_{k'}\max_{k}\{\pi_{kk'}\}$. Obviously, $ECR({\mathcal{C}})=0$ if ${\mathcal{C}}=\mathcal{C}^*$.

Theorem \ref{optimial_partition_theory} is stated as follows.

\begin{theorem}\label{optimial_partition_theory}
If the data matrix $\bX=(x_{ij})_{n\times p}$ is generated according to (\ref{step1}) and (\ref{step2}), then
%and $\bX_j,j=1,\dots,p$ is the $j^{th}$ column of $\bX$ representing a feature.
\begin{itemize}
\item[(I)] for any $p^*<j\leq p$, the $j^{th}$ feature is a noise feature, and for any $\ 1\leq j\leq p^*$ the $j^{th}$ feature is a relevant feature.
\item[(II)] there holds
\begin{equation}
\mathbb{E}[a_j(\mathcal{C}^*)]  \left\{\begin{array}{cc}
  >K-1 & 1\leq j\leq p^* \\
  =K-1&  \text{otherwise}
\end{array}\right..
\end{equation}
\item[(III)] the natural partition $\mathcal{C}^*$ of $\bX$ is its optimal partition. Furthermore, the optimal partition of $\bX$ satisfies
\begin{eqnarray}\label{sigle_optimal}
\mathcal{C}^*= \arg\max_C \mathbb{E}[a_j(\mathcal{C})],\ \forall 1\leq j\leq p^*.
\end{eqnarray}
\end{itemize}
\end{theorem}

We present some comments on Theorem \ref{optimial_partition_theory} as follows.
%Let us make a few comments for Theorem \ref{optimial_partition_theory} as following:
\begin{itemize}
\item Theorem \ref{optimial_partition_theory} (I) shows the existence of noise features in a very general situation. Thus, the Definition \ref{defoptimal} and \ref{noise} make sense. We notice that by Definition \ref{noise}, noise features are those on which samples from any partitions have the same expectation. This characteristics of noise feature can be used to select features for high-dimensional data clustering problems.
%To the best of our knowledge, it should be observed that the noise feature is so precisely defined in the first time.
This is a direct consequence of the new definition of the optimal partition (Definition \ref{defoptimal}) which cannot be resulted from the traditional formulation (\ref{risk}).

\item Theorem \ref{optimial_partition_theory} (II) reveals that the expectation of relevant features and noise features have a significant gap, which then underlies the distinguishability of the relevant features and noise features in applications. For example, the $\ell_1$-k-means proposed by Witten and Tibshirani~\cite{witten2010} works actually based on the use of such gap information. In fact, according to~\cite{witten2010}, given an estimated partition $\hat{\mathcal{C}}$, the $\ell_1$-k-means defines the optimal feature weight
\begin{eqnarray}\label{soft}
\hat{\bw}=\dfrac{S(a(\hat{\mathcal{C})},\Delta)}{\|S(a(\hat{\mathcal{C}}),\Delta)\|_2},
\end{eqnarray}
where $S(a,\Delta)_j=\max(a_j-\Delta,0)$ defined by soft thresholding function. From (\ref{soft}), it is clear that any feature corresponding to $a_j<\Delta$ has been identified as a noise feature, otherwise, a relevant feature. Since $\hat{\mathcal{C}}$ is considered as an approximation of the optimal partition $\mathcal{C}^*$, $a_j(\hat{\mathcal{C}})$ can be viewed as an approximation to $\mathbb{E}[a_j(\mathcal{C}^*)]$. It follows from (\ref{soft}) that the $\ell_1$-k-means performs feature selection actually by making use of (II) of Theorem \ref{optimial_partition_theory}. We will later show that the $\ell_0$-k-means algorithm we suggested in Algorithm \ref{alg-l0} also follows the same principle.

\item Theorem \ref{optimial_partition_theory} (III) indicates that in the very general case, $\mathcal{C}^*= \arg\max_C \mathbb{E}[a_j(\mathcal{C})]$ for all relevant features. This means the optimal partition is the partition that maximizes the BCSS values on each relevant features. Based on this, we have
\begin{eqnarray}
\max_{\mathcal{C}}\sum_{j=1}^p\mathbb{E}[a_j(\mathcal{C})]&=&\sum_{j=1}^p\mathbb{E}[a_j(\mathcal{C}^*)]\\\nonumber
&=&\sum_{j=1}^{p^*}\mathbb{E}[a_j(\mathcal{C}^*)]+(p-p^*)(K-1)\\\nonumber
&=&\sum_{j=1}^{p^*}\max_{\mathcal{C}}\mathbb{E}[a_j(\mathcal{C})]+(p-p^*)(K-1).
\end{eqnarray}
These equations are of special significance for the high-dimensional clustering because these equations reveal that the defined optimal partition $\mathcal{C}^*$ does not vary when dimension $p$ (number of features) varies with the number of samples $n$. This support that taking (\ref{optimialpartion}) as a definition of optimal partition rather than (\ref{risk}) is reasonable and it is consistent with our intuition.

\end{itemize}

The above expositions support that the new definitions on optimal partition and noise feature introduced in Definition 1 and 2 are of special significance when dealing with high-dimensional clustering problems. Based on these definitions, we will propose the $\ell_0$-k-means and analyze its theoretical properties below.

\subsection{A New Sparse Clustering Framework and $\ell_0$-k-means}
As a common practice, the $\ell_1$ penalty can be replaced by any $\ell_q (0\leq q<1)$ penalty in sparse modeling if a more sparse result hopes to be obtained~\cite{xuchang2012,marjanovic2012}. This is, however, by no means trivial and tractable for sparse clustering problems. For example, if we use $\ell_0$ penalty to replace the $\ell_1$ penalty in (\ref{l1framework}) which then leads to the following optimization problem:
\begin{eqnarray}\label{l0compare}
&\max\limits_{\bw, \Theta(\mathcal{C})}&\sum_{j=1}^pw_jf_j(\bX_j,\Theta(\mathcal{C}))\\\nonumber
&\text{s.t.}&\ \|\bw\|_{2}\leq 1,\ \|\bw\|_0\leq s, w_j\geq 0,\ \forall j.
\end{eqnarray}
This model is difficult to analyze and compute.

To overcome this difficulty, we propose in the present research to jointly apply the $\ell_{\infty}$ and $\ell_0$ penalty. In other words, we suggest to use the following new sparse clustering framework:
\begin{eqnarray}\label{l0framework}
&\max\limits_{\bw, \Theta(\mathcal{C})\in D}&\Big\{\sum_{j=1}^p w_jf_j(\bX_j,\Theta(\mathcal{C}))\Big\}\\\nonumber
&\text{s.t.}&\ \|\bw\|_{\infty}\leq 1,\ \|\bw\|_0\leq s,\ w_j\geq 0,\ \forall j,
\end{eqnarray}
where $ \|\bw\|_{\infty}=\max\limits_{i=1,2,\dots,p}|w_j|$ and $\|\bw\|_0$ is the number of nonzero components of $\bw$. We will show that, surprisingly, the new sparse clustering framework (\ref{l0framework}) is not only tractable, but can be analyzed theoretically as well.

The difficulty of solving the new sparse clustering framework (\ref{l0framework}) mainly comes from the existence of two different types of variables: the partition variable $\mathcal{C}=\{C_1,\dots,C_K\}$ featured by clustering the data set into $K$ groups, and the feature weight variable $\bw=(w_1,\dots,w_p)^\top$ that characterizes which features are responsible for the valid clustering. To tackle such difficulty, we suggest to apply the well-known alternative iteration technique. That is, we will solve (\ref{l0framework}) iteratively through two steps: First, fix $\bw$ and solve the problem (\ref{l0framework}) with respect to $\mathcal{C}$, and then, fix $\mathcal{C}$ and solve the problem with respect to $\bw$. This procedure is recursively repeated until a stopping criterion is satisfied. Thus, the sparse framework (\ref{l0framework}) can be formally solved by the procedure defined as the following:
\begin{itemize}
\item[(i)] Initialize $\bw^{0}=(w_1^0,\dots,w_p^0)=(1,\dots,1)^\top$ and $\bw^{1}=(w_1^1,\dots,w_p^1)=\frac{1}{\sqrt{p}}(1,\dots,1)^\top$. Let $t:=1$, for any $t\geq 0$ do the following steps (ii) and (iii) until
$$\frac{\sum_{j=1}^p|w_j^{t}-w_j^{t-1}|}{\sum_{j=1}^p|w_j^{t-1}|}<10^{-4}.$$
\item[(ii)] Let $f_j(\bX_j,\Theta(\mathcal{C}))\leftarrow w_j^{t}f_j(\bX_j,\Theta(\mathcal{C}))$, and then find the partition $\mathcal{C}^{t}$ by applying any clustering method (according to (\ref{framework})).
\item[(iii)] Let $f_j(\bX_j,\Theta(\mathcal{C}))\leftarrow f_j(\bX_j,\Theta(\mathcal{C}^{t}))$, solve the optimization problem
\begin{eqnarray}
&\max\limits_{\bw}& \sum_{j=1}^pw_jf_j(\bX_j,\Theta(\mathcal{C}^{t}))\\\nonumber
&s.t.& \|\bw\|_{\infty}\leq 1, {\|\bw\|_0\leq s}, w_j\geq 0,
\end{eqnarray}
to get $\bw^{t+1}$. Set $t:=t+1$.
\end{itemize}

In the above procedure, the step (ii) can be solved by any well-developed clustering algorithm as long as its formulation can be subsumed into the framework (\ref{framework}). Thus, the mainly computational complexity of the procedure comes from the step (iii). We will handle the step (iii) for a specific realization, that is the following $\ell_0$-k-means model.

Like $\ell_1$-k-means, we define a clustering model by specifying $f_j(\bX_j,\Theta(\mathcal{C}))$ in (\ref{l0framework}) to be the $a_j$ defined as in (\ref{aj}). Thus, the $\ell_0$-k-means we suggest is modeled as follows:
\begin{eqnarray}\label{l0kmeans}
&\max\limits_{\mathcal{C}, \bw}&\sum_{j=1}^pw_j\Big(\frac{1}{n}\sum_{i=1}^n\sum_{i'=1}^nd_{ii'j}-\sum_{k=1}^K\frac{1}{n_k}\sum_{i,i'\in C_k}d_{ii'j}\Big)\\\nonumber
&\text{s.t.}&\ \|\bw\|_{\infty}\leq 1,\ \|\bw\|_0\leq s,\ w_j\geq 0,\ \forall j.
\end{eqnarray}

In order to solve the $\ell_0$-k-means by the above procedure, we have to deal with the step (iii), that is
\begin{eqnarray}\label{l0subproblem}
&\max\limits_{\bw}& \bw^\top {\bf a}\\\nonumber
&s.t.& \|\bw\|_{\infty}\leq 1, {\|\bw\|_0\leq s}, w_j\geq 0.
\end{eqnarray}
We will prove the following Theorem \ref{theorem1} to solve (\ref{l0subproblem}). %Here, we need to mention that all the technical proof will be presented in Appendix (See Section \ref{appendix}).

\begin{theorem}\label{theorem1}
When the sequence $\{a_j\}_{j=1}^p$ defined in (\ref{aj}) is decreasingly ordered and non-identical, i.e., $a_i\geq a_j$ for any $i<j$, an optimal solution of (\ref{l0subproblem}) is given by
\begin{equation}\label{l0solution}
{w}_j^*=\left\{\begin{array}{cc}
 1& j\leq \lfloor s \rfloor \\
  0&j> \lfloor s \rfloor
\end{array}\right.,
\end{equation}
where $ \lfloor s \rfloor$ means the integer part of s.
\end{theorem}
Based on Theorem \ref{theorem1}, if we decreasingly order the $a_j'$s, then the solution of (\ref{l0subproblem}) can be directly set as (\ref{l0solution}), that is, we can directly assign $w_j=1$ for the components corresponding to the first $\lfloor s\rfloor$ elements of $\{a_j\}_{j=1}^p$ and $w_j=0$ otherwise. This procedure can be seen the $\ell_0$-k-means selects the relevant features by means of the gap information we discussed in Theorem \ref{optimial_partition_theory} (II). Note that the formulation (\ref{l0solution}) can be viewed as performing the {\it hard thresholding operation}~\cite{blumensath2008}, similar to the {\it soft}~\cite{witten2010} and {\it half}~\cite{xuchang2012} thresholding operations used in $\ell_1$-k-means and the $\ell_{1/2}$ regularization approach respectively. Finally, we suggest the following $\ell_0$-k-means algorithm for sparse clustering.
\begin{algorithm}[htb] %算法的开始

\renewcommand{\algorithmicrequire}{\textbf{Input:}}

\renewcommand\algorithmicensure {\textbf{Output:} }

\caption{$\ell_0$-k-means algorithm} %算法的标题

\label{alg-l0} %给算法一个标签，这样方便在文中对算法的引用

\begin{algorithmic}[1] %这个1 表示每一行都显示数字

\REQUIRE ~~\\ %算法的输入参数：Input

Cluster number $K$ and data matrix $\bX$.\\

\ENSURE ~~\\ %算法的输出：Output

Clusters ${C}_1,{C}_2,\dots,{C}_K$ and $\bw^{new}$.\\
~\\
\STATE $w_1^{new}=w_2^{new}=\dots=w_p^{new}=\frac{1}{\sqrt{p}}$.\\
\STATE Let $\bw^{old}=\bw^{new}$. Transform $d_{ii'j}\leftarrow w_j^{old}d_{ii'j}$. Find clusters ${C}_1,{C}_2,\dots,{C}_K$ based on standard k-means.\\

\STATE Fix ${C}_1,{C}_2,\dots,{C}_K$. Calculate $a_j=\frac{1}{n}\sum_{i,i'}^nd_{ii'j}-\sum_{k=1}^K\frac{1}{n_k}\sum_{i,i'\in C_k}d_{ii'j}$. Order the $a_j'$s decreasingly, then assign $w_j=1$ for the components corresponding to the top $\lfloor s\rfloor$ elements of $\{a_j\}_{j=1}^p$ and $w_j=0$ otherwise.

\STATE Repeat step 2 and 3 until
$$\frac{\sum_{j=1}^p|w_j^{new}-w_j^{old}|}{\sum_{j=1}^p|w_j^{old}|}<10^{-4}.$$
\end{algorithmic}

\end{algorithm}

Observe that the standard k-means costs $O(nKp)$ time in complexity, while the step 3 of Algorithm \ref{alg-l0} costs $O(p\lfloor s\rfloor )$ in time, so the suggested $\ell_0$-k-means algorithm is an $O(nKp)$ ( if $\lfloor s\rfloor\leq nK$) complexity method which is the same as the standard k-means. The condition $\lfloor s\rfloor\leq nK$ is reasonable, because it is often assumed the number of the relevant features in high-dimensional data clustering problems is only a small portion of features. Therefore, the $\ell_0$-k-means should be very efficient in implementation. This is supported in the simulations of Section \ref{simulations} below.

\subsection{Theoretical Analysis of $\ell_0$-k-means}\label{consistensection}
%\label{theoreticalanalysis}

In this subsection, we assess the theoretical properties of the proposed $\ell_0$-k-means. The main conclusion is that under mild conditions, the solution of $\ell_0$-k-means algorithm has a feature selection consistency property if the data matrix is generated from a high-dimensional Gaussian mixture model, namely, $\mathcal{F}_k(\mu_k,\Sigma_k)=\mathcal{N}_k(\mu_k,\Sigma_k)$ in (\ref{step1}) and (\ref{step2}).

Now let us consider the consistency of $\ell_0$-k-means. Since there are two main steps in Algorithm \ref{alg-l0}, we need to consider each step separately. We notice that in the first step, the algorithm seeks for partition $\hat{\mathcal{C}}$ via maximizing BCSS, while in the second step, it selects relevant features based on $\hat{\mathcal{C}}$. For the partition step, we can prove the following result.
\begin{theorem}\label{theorem2}
(Partition Consistency) Suppose the data matrix $\bX\in\mathbb{R}^{n\times p}$ is generated from the Gasussian mixture model by (\ref{step1}) and (\ref{step2}), $\mathcal{C}^*$ is the optimal partition of $\bX$ and $\hat{\mathcal{C}}\in\arg\max\limits_{\mathcal{C}}\sum_j a_j(\mathcal{C})$. Then
\begin{eqnarray}\label{probability_partition}
\mathbb{P}(ECR(\hat{\mathcal{C}})\geq F(p^*)|\mathcal{C}^*)\leq 2K^{-n},
\end{eqnarray}
if $p\leq p^*n,p^*\geq\kappa\triangleq 128\frac{K+\sum_k\pi_k\mu_k^2}{(\sum_k\pi_k\mu_k^2-(\sum_k\pi_k\mu_k)^2)^2},$ where $F(\cdot)$ is a decreasing function such that $F(\kappa)=1-\max_k\pi_k,F(+\infty)=0$.

\end{theorem}

From Theorem \ref{theorem2}, we can conclude that $ECR(\hat{\mathcal{C}}|\mathcal{C}^*)\stackrel{\mathbb{P}}{\rightarrow}0$ if  $p^*\rightarrow \infty$, $p\leq p^*n$ and $n\rightarrow\infty$. {This first condition is necessary because even if we know $\mu_k$ and all relevant features, it is still needed to have partition consistency.} This second condition ($p\leq p^*n$), however, might not be necessary. But it is at least necessary for $p$ to satisfy $p=O(p^{*}n^2)$, which shows that the conditions cannot be relaxed too much whenever possible. The case $p=O(p^{*}n^2)$ can be obtained by considering the possibility $\mathbb{P}(ECR(\mathcal{C})>\epsilon|\mathcal{C}^*)$ for any arbitrary partition $\mathcal{C}$. Whenever this condition is violated, we can construct a parameter settings $p,p^*$ and $n$ such that $\mathbb{P}(ECR(\mathcal{C})>\epsilon|\mathcal{C}^*)\not\rightarrow 0$. That is to say, the estimation (\ref{probability_partition}) is optimal in certain sense.

For the relevant feature selection step, we will prove the following Theorem \ref{theorem3}.

\begin{theorem}\label{theorem3}
In the setting of Theorem \ref{theorem2}, if $p=o(\exp\{\rho n\})$ ($\rho=\frac{\sum_{k}\pi_{k}\mu_k^2}{258}$ is a constant), and $p\geq M$ ($M$ is a constant depends on $\mu_k,\pi_k, k=1,\dots,K$), then
\begin{eqnarray}\label{probabillity_featuregap}
\mathbb{P}(\min_{j\leq p^*}a_j(\hat{\mathcal{C}})>\max_{p^*< j\leq p}a_j(\hat{\mathcal{C}})|\mathcal{C}^*)\rightarrow 1\ \text{as }n\rightarrow\infty
\end{eqnarray}
%$$\mathbb{P}(\min_{j\leq p^*}a_j(\hat{\mathcal{C}})>\max_{p^*< j\leq p}a_j(\hat{\mathcal{C}})|\mathcal{C}^*)\rightarrow 1\ \text{as }n\rightarrow\infty$$
where $a_j$ is the BCSS of $j^{th}$ feature defined as in (\ref{aj}).
\end{theorem}
The estimation (\ref{probabillity_featuregap}) in Theorem \ref{theorem3} shows essentially that the gap between the relevant features and noise feature (see (II) of Theorem \ref{optimial_partition_theory}) will be kept probability. While, Theorem \ref{theorem2} shows $ECR(\hat{\mathcal{C}})$ less than or equal to any small positive constant with high probability. Thus, combing the conditions of Theorem \ref{theorem2} and \ref{theorem3}, we can establish the feature selection consistency property of the solution of $\ell_0$-k-means algorithm. The details are presented as follows.

\begin{theorem}(Feature Selection Consistency)\label{consistent}
Suppose the data matrix $\bX\in\mathbb{R}^{n\times p}$ is generated from the Gasussian mixture model by (\ref{step1}) and (\ref{step2}) with the properties
$p\leq p^*n, p^*\geq\kappa$ and $p=o(\exp\{\rho n\})$  ($\rho=\frac{\sum_{k}\pi_{k}\mu_k^2}{258}$), and $\bw^*$ is the solution of Algorithm \ref{alg-l0}. Then
$$\mathbb{P}(w^*_j=1,j\leq p^*\ and\ w^*_j=0, j>p^* |\mathcal{C}^*)\rightarrow 1\ \text{as }n\rightarrow\infty.$$
\end{theorem}
Let us make some remarks on Theorem \ref{consistent} as follows.
\begin{itemize}
\item If $p^*$ is fixed (i.e., the number of relevant features is fixed), Theorem \ref{consistent} holds if $p$ and $p^*$ satisfies the relation $p\leq p^*n$, or equivalently saying, $p$ grows at the same order of $n$.

\item If $p^*$ is not fixed but varies proportional to $p$, the conditions of Theorem \ref{consistent} degenerate to be $p=o(\exp\{\rho n\})$, i.e., the number of features grows slower than the exponential growth of the sample size. We notice that such condition is optimal when an ultra-high dimensional feature selection problem is dealt with in the penalized regression approach (see, e.g., ~\cite{zhao2006model,wainwright2009sharp,fan2010}).
%This similar condition was widely used in high-dimensional statistics problem.

\item Theorem \ref{consistent} assumes the data matrix $\bX$ is generated from a Gaussian mixture model. Actually, this condition can be generalized to any kind of subgaussian distributions.
\end{itemize}

From Theorem \ref{consistent}, we conclude that under suitable conditions the solution of $\ell_0$-k-means algorithm defined by Algorithm \ref{alg-l0} has the feature selection consistency property.
%Such consistency result is the first time to be verified for sparse clustering approaches in the high-dimensional setting.
In particular, no feature selection consistency result has been justified for the $\ell_1$-k-means. This reveals a difference and potential advantage of the new suggested $\ell_0$-k-means for high-dimensional data clustering problems.

\section{Experimental Evaluation}\label{simulations}
In this section, we evaluate and compare the performance of the $\ell_0$-k-means, $\ell_1$-k-means and standard k-means based on a set of synthetic data and a concrete Allen Developing Mouse Brain Atlas data set.

The $\ell_0$-k-means and $\ell_1$-k-means algorithms involve a tunning parameter $s$, controlling the sparsity of the features selected. Witten and Tibshirani~\cite{witten2010} has conducted a strategy to select the best tunning parameter $s$ successfully based on Gap Statistics~\cite{tibshirani2001}. Thus, we employ this same strategy for the proposed $\ell_0$-k-means as well. Four different criteria are taken for a more comprehensive comparison for all the algorithms. The first criterion is the {\it Classification Error Rate} (CER) used in~\cite{witten2010,chipman2006}, which was used instead of the Error Clustering Rate (ECR) adopted early in order to make the new algorithm ($\ell_0$-k-means) directly comparable with the results offered in Witten and Tibshirani's original paper~\cite{witten2010}. CER is defined as  $ CER\triangleq \sum_{i>i'}|1_{\hat{\mathcal{C}}(i,i')}-1_{\mathcal{C}^*(i,i')}|/{n\choose 2}$, where $1_{{\mathcal{C}}(i,i')}$ is the indicator function if the $i^{th}$ and $j^{th}$ samples are in the same group with respect to partition $\mathcal{C}$.
The second criterion is the number of {\it non-zero weights} NW=${|\{i:\hat{w}_i\neq 0\}|}$, where $\hat{\bw}$ is any estimation of $\bw$ yielded by each compared algorithm. It measures how many features are selected as relevant features by each algorithm. The third criterion is the number of {\it proper zero weights} PZW=${|\{i:w_i=0,\hat{w}_i=0\}|}$, which measures how many noise features are correctly eliminated by an algorithm. The fourth criterion is the number of {\it proper nonzero weights} PNW=${|\{i:w_i\neq 0,\hat{w}_i\neq 0\}|}$, which measures how many relevant features are correctly selected by an algorithm. Note that PZW and PNW together measure the capability of an algorithm that correctly include the relevant features and exclude the noise features, while CER measures the mistaken classification rate. These criteria can fairly characterize the performance of each compared algorithms.

\subsection{Evaluation on Synthetic Data}
We conducted four sets of experiments to evaluate the performance of each algorithm. The first experiment was to verify that Gap Statistics can be used to select an appropriate tunning parameter for the $\ell_0$-k-means. The second experiment was designed to detailedly compare the performance of $\ell_0$-k-means, $\ell_1$-k-means and standard k-means. The third experiment then compared the $\ell_0$-k-means with several other related well-known clustering methods. In these three experiments, the respective algorithms were all implemented in the circumstance that all the features are uncorrelated. In the fourth experiment, however, we compared the algorithms under the circumstance that some features are correlated. This was designed to assess the influence of correlation among that features to the performance of the algorithms.

{\it Experiment 1:}
We suppose there exists 6 clusters and each contains 20 samples in data matrix $\bX_{120\times 2000}$. There are 2000 features among which the first 200 features are relevant features. For the $k^{th}$ cluster, relevant features are sampled from $\mathcal{N}(0.5\cdot k,1)$ and noise features are sampled from $\mathcal{N}(0,1)$ independently. The data matrix is normalized before using any algorithm and the experiment is conducted 20 times for $\ell_0$-k-means and standard k-means. All the results are shown in Figure \ref{Gaps}.

Figure \ref{Gaps} summarizes all the results of $\ell_0$-k-means compared with standard k-means. From the first plot, we can see that the best tunning parameter of $\ell_0$-k-means has been selected by Gap Statistics (because the value of horizontal axis which is corresponding to maximal Gap Statistics is around 200). The middle one shows that the obtained partition has a significant smaller CER compared with standard k-means. The third plot plots the averaged estimated weights for all features. From that, we can find that the averaged estimated weights for relevant features are generally close to 1 while the estimated weights for noise features are close to 0. This shows that by using Gap Statistics the $\ell_0$-k-means does have good feature selection capacity and thus gives more accurate partitions.

\begin{figure}[htp]
  \centering
  \includegraphics[width=.3\textwidth]{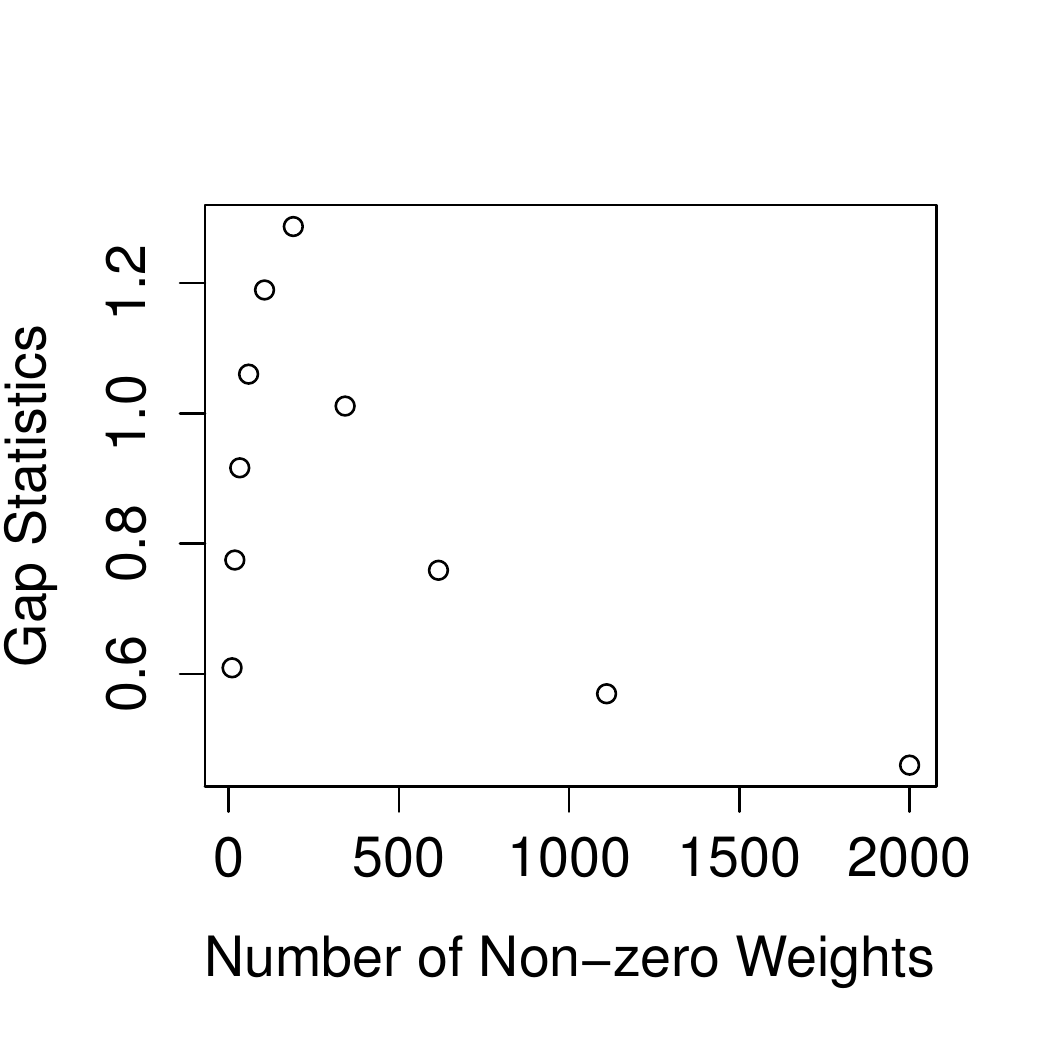}
  \includegraphics[width=.3\textwidth]{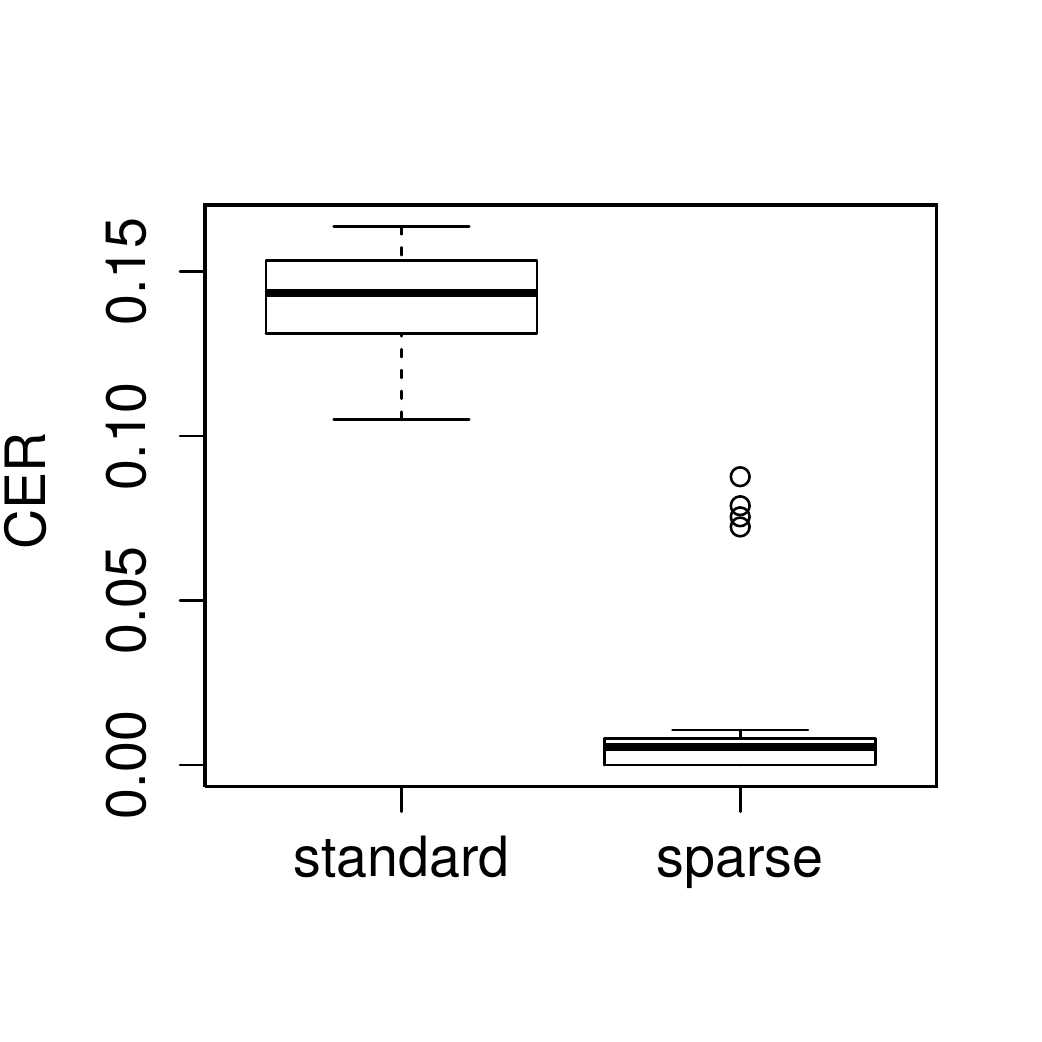}
  \includegraphics[width=.3\textwidth]{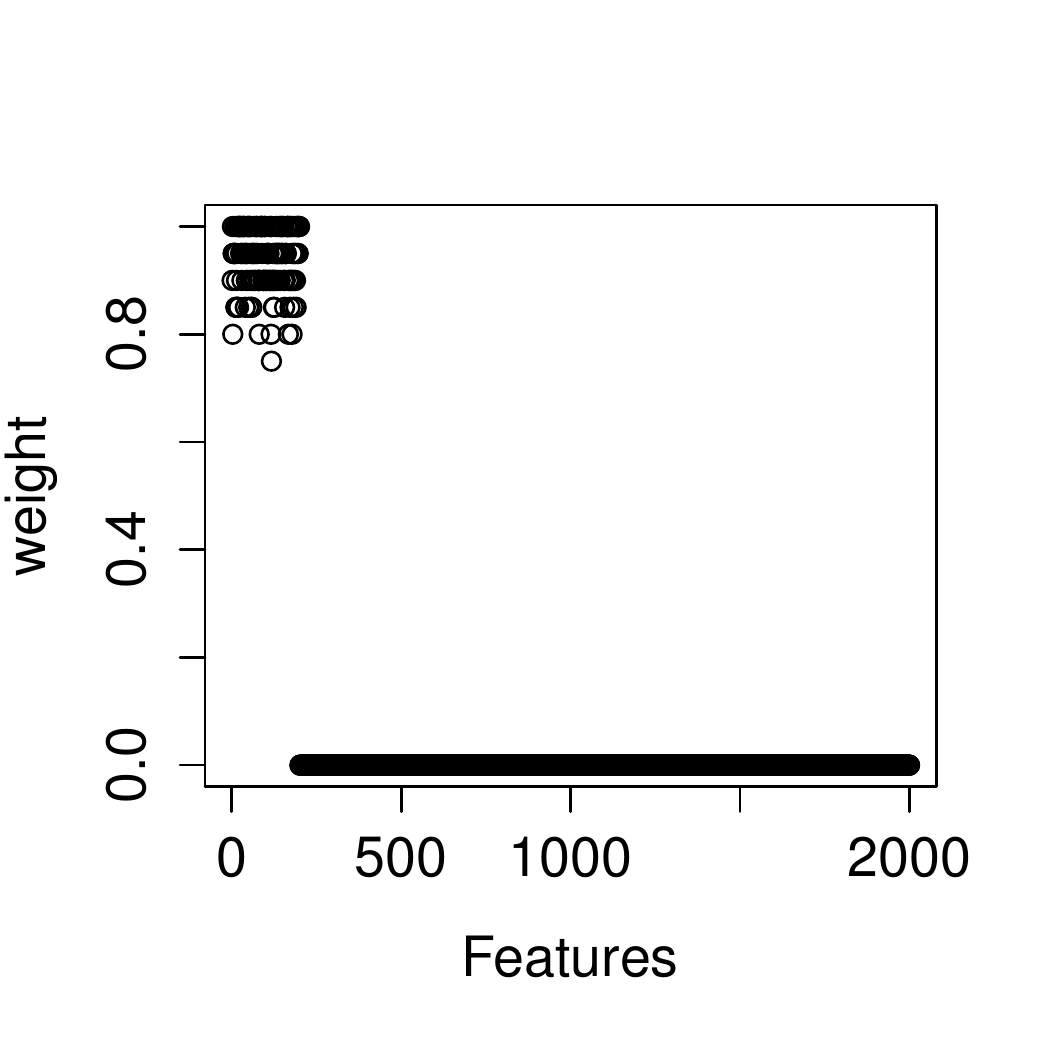}
  \caption{Overview of $\ell_0$-k-means.}\label{Gaps}
\end{figure}

{\it Experiment 2:} We evaluate standard k-means, $\ell_0$-k-means and $\ell_1$-k-means by varying the number of feature $p$. We took the data matrix $\bX_{60\times p}$ from the data generating processing (\ref{step1}) and (\ref{step2}) with respect to $K=3$. For (\ref{step2}), we assume elements $x_{ij}\sim\mathcal{N}(\mu_{ij},1)$ in $\bX$ independently, where
\begin{equation}
\mu_{ij} = \left\{\begin{array}{cc}
  \mu & \text{if}\ i\in C_1,j\leq 50 \\
  -\mu& \text{if}\ i\in C_2,j\leq 50\\
  0& \text{if}\ i\in C_3,\ \text{or}\ j>50
\end{array}\right..
\end{equation}
Then, the first $50$ features are relevant features while the others are noise features according to Theorem \ref{optimial_partition_theory}. We suppose each cluster contains 20 samples and set $\mu=0.6,0.7$, $p=200,500,1000$ and each simulation is repeated 20 times. The averaged experimental results are shown in Figure \ref{CERS1} and Table \ref{WS1}.

\begin{table}[ht]
  \centering
  \caption{Mean values (PZW, PNW) for different $\mu$ in experiment 2.}\label{WS1}
    \begin{tabular}{rrrr}
    \toprule
    \multicolumn{4}{c}{$\mu=0.6$} \\
    \midrule
          & k-means & $\ell_1$-k-means & $\ell_0$-k-means \\
    p=200 & (0, 50) & (1.6, 50)  & (138.9, 33.8)  \\
        p=500 & (0, 50) & (214.4, 49)  & (440.1, 30.8) \\
    p=1000 & (0, 50) & (618.1, 42.6)  & (941.6, 32.8) \\
    \midrule
    \multicolumn{4}{c}{$\mu=0.7$} \\\midrule
    p=200 & (0, 50) & (0, 50)& (140.9, 33.7) \\
    p=500 & (0, 50) & (295.5, 49)  & (440.4, 34.9) \\
    p=1000 & (0, 50) & (685.7, 47.5) & (937.3, 31.35)\\
    \bottomrule
    \end{tabular}%
  \label{tab:CER}%
\end{table}%

\begin{figure}
\begin{center}
  \includegraphics[width=350pt]{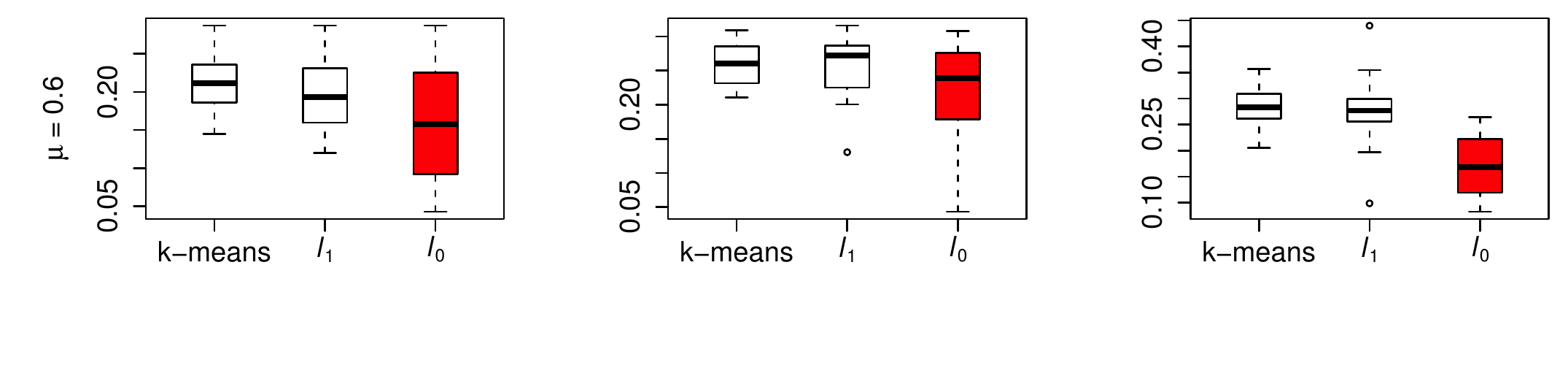}\\
  \includegraphics[width=350pt]{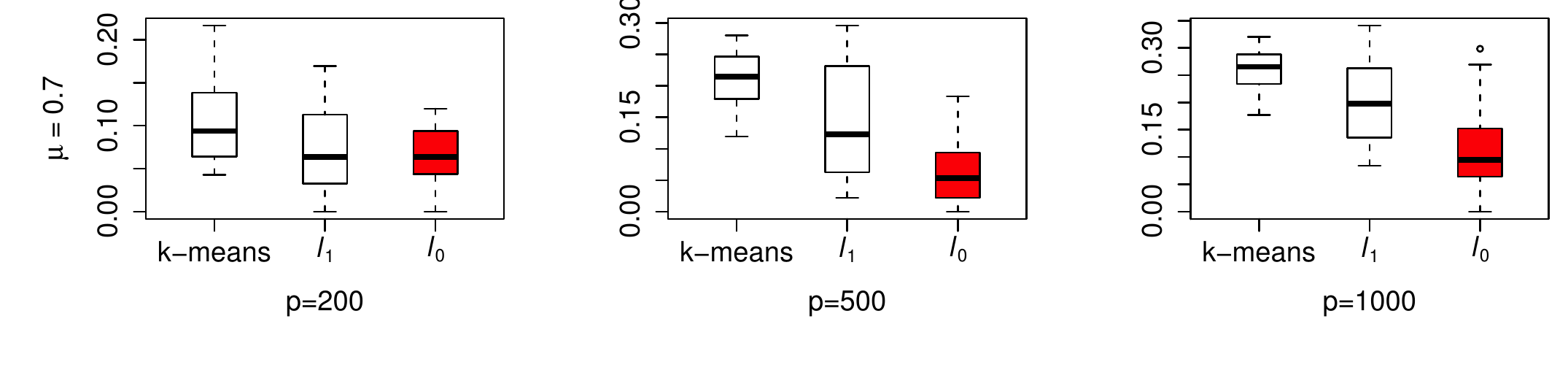}
\caption{\emph{Boxplot of CER for different combinations of $\mu$ and $p$.}}\label{CERS1}
  \end{center}
\end{figure}
From Figure \ref{CERS1}, we can observe that the mean CER values of standard k-means are always greater than those of sparse k-means. Therefore, clustering capacity of sparse k-means tends to outperform k-means when the noise features exist. Moreover, the boxplot shows that $\ell_{0}$-k-means almost has the best clustering performance due to the smallest mean of CER. In order to explain this phenomenon, Table \ref{WS1} shows other criteria. From the table, we can find that although most relevant features are found by $\ell_1$-k-means, it keeps more noise features than the $\ell_0$-k-means. Particularly, the $\ell_1$-k-means completely failed when $p=200$ because it kept all the noise features in the clustering results. Note that this is coherent with the results offered in Witten and Tibshirani's original paper \cite{witten2010}. On the contrary, $\ell_0$-k-means can detect more noise features than $\ell_1$-k-means at the price of eliminating a little more relevant features as well. This compromised property of $\ell_0$-k-means might result in a lower CER than the $\ell_1$-k-means.

{\it Experiment 3:} In this experiment, the proposed sparse k-means is compared with penalized log-likelihood approach~\cite{pan2007} (P-likelihood for short) and PCA followed by k-means~\cite{chang1983} (PCA-k-means for short). The data is generated similarly to the previous Experiment 2 with different parameter values. In this part, we consider two data sets of different sizes. The first data set consists of 3 clusters containing 10 samples each. Each sample has 25 features and 5 of them are relevant with $\mu=1$. The second data set consists of 3 clusters containing 20 samples each. Each sample has 500 features and 50 of them are relevant with $\mu=0.7$. The simulation is repeated 20 times. The averaged experimental results are shown in Table \ref{simutable2}.
\begin{table}[!h]
\tabcolsep 0pt \caption{Mean and standard deviation of CER, mean values of PZW and PNW for different models in Experiment 3.}\label{simutable2}
\vspace*{-6pt}
\begin{center}
\def\temptablewidth{1\textwidth}
{\rule{\temptablewidth}{1pt}}
\begin{tabular*}{\temptablewidth}{@{\extracolsep{\fill}}ccccc}
Simulation & Method  & CER & Mean of PZW & Mean of PNW \\
\hline
$p=25,\mu=1$, & k-means & 0.312(0.001)  &  0 & 5    \\
     & $\ell_1$-k-means & 0.308(0.003) & 12& 4.3 \\
     & $\ell_0$-k-means & {\bf 0.299(0.002)} & {\bf 10.65} & {\bf 4.6} \\
     & PCA-k-means & 0.333(0.003) & 0 & 5 \\
     & P-likelihood& 0.301(0.002) & 9.5 & 5 \\
$p=500,\mu=0.7$, & k-means & 0.237(0.001)  &  0 & 50    \\
          & $\ell_1$-k-means & 0.171(0.005) & 315& {\bf 49.2} \\
          & $\ell_0$-k-means & {\bf 0.058(0.002)} & {\bf 444.7} & 34.7 \\
          & PCA-k-means &{0.103(0.003)} & 0 & 50 \\
          & P-likelihood& 0.168(0.003) & 424.3 & 37.3 \\
       \end{tabular*}
       {\rule{\temptablewidth}{1pt}}
       \end{center}
       \end{table}

Let us make a few comments for Table \ref{simutable2}. First of all, CERs of PCA-k-means and standard k-means are higher compared with the $\ell_0$-k-means and $\ell_1$-k-means. The reason is that principal components are linear combinations of all features (includes noise features) and k-means treats all features equally. Thus, the noise features influence the clustering capacity of two methods dramatically. Since the penalized model-based clustering method of Pan and Shen~\cite{pan2007} also considers the noise features of clustering data, it resulted in relatively low CER in the simulation.
For the suggested $\ell_0$-k-means, it generally achieved better clustering performance (CER) than other comparable models. It is because $\ell_0$-k-means has eliminated most noise features in the simulation (See PZW).

{\it Experiment 4:}
For previous synthetic data, we assume different features are independent, while for real life applications, this is often not true. In order to validate our algorithms in broader settings, we consider situations when different features are correlated. Suppose $\bx_i\sim \mathcal{N}(\mu,\Sigma)$, where $\Sigma_{ab}=\rho^{|a-b|}$. Similar to experiment 1, suppose there exists 6 clusters and each contains 20 samples. There are 2000 features among which 200 are relevant. For the $k^{th}$ cluster, its centroid is $1\cdot k$ on relevant features and $0$ on noise features. The experiment is conducted 20 times.
As shown in Figure \ref{relevant} and Table \ref{tab:weight}, $\ell_0$-k-means still performs better than $\ell_1$-k-means.

From the experiments 1-4, we can conclude that the $\ell_0$-k-means algorithm generally outperforms the $\ell_1$-k-means and standard k-means in generating lower CER and, in particular, the $\ell_0$-k-means has an obvious stronger capability of eliminating the noise features than the $\ell_1$-k-means.
\begin{figure}
  \centering
  \includegraphics[width=.45\textwidth]{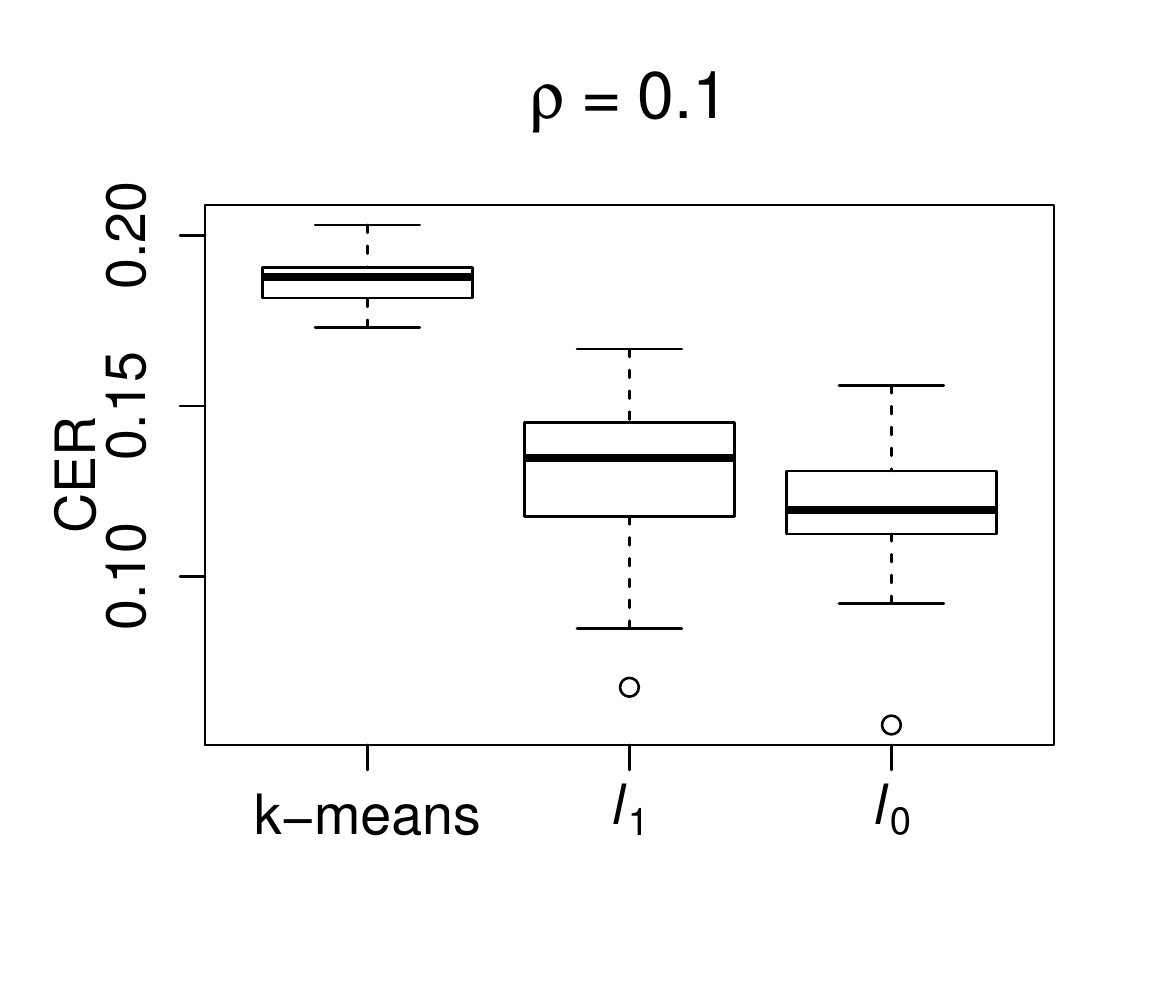}
  \includegraphics[width=.45\textwidth]{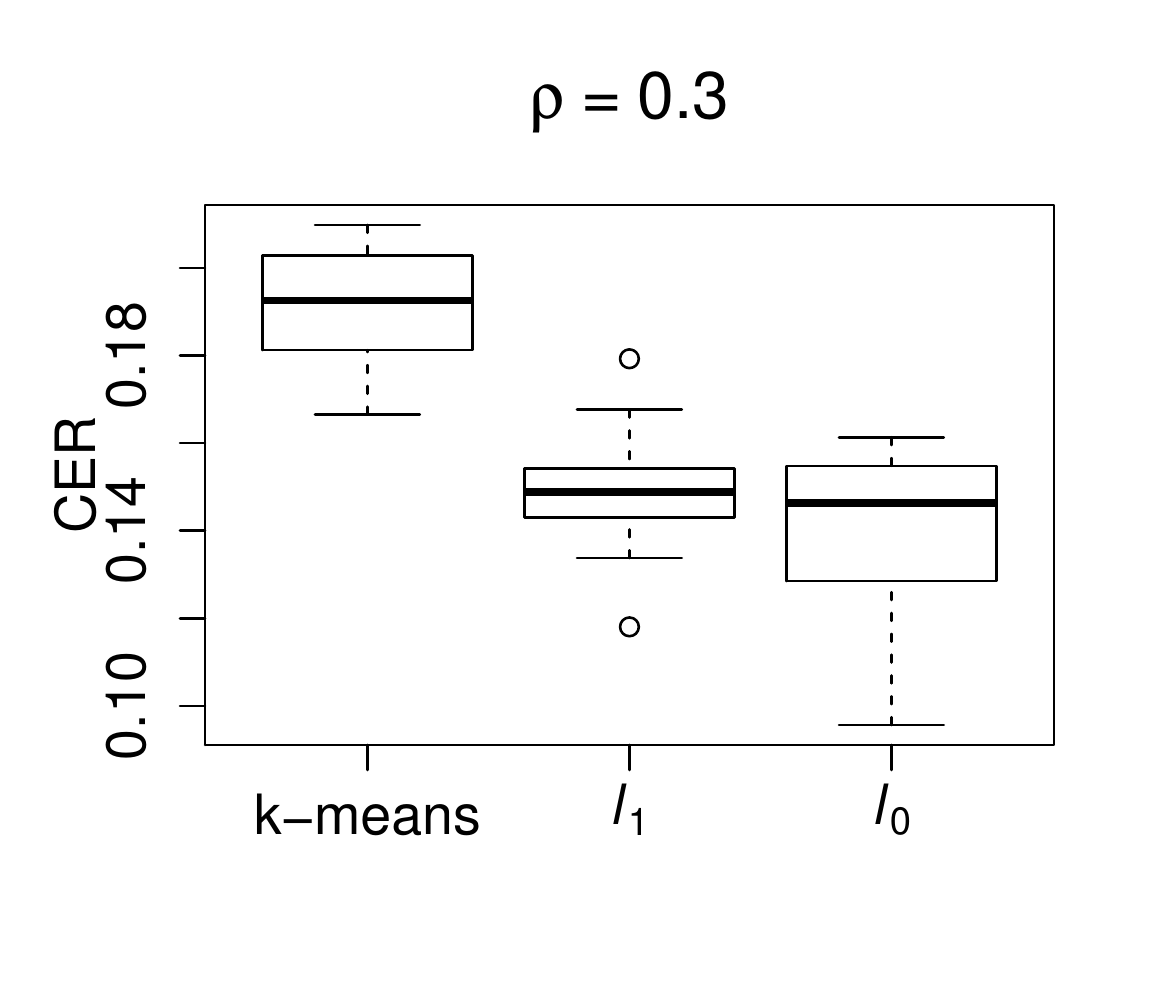}
  \caption{CER for data with correlated features. The left one is $\rho=0.1$ and the right one is $\rho=0.3$.}\label{relevant}
\end{figure}

\begin{table}[ht]
  \centering
  \caption{Mean value and standard deviation of PZW and PNW for different $\rho$.}\label{WS4}
    \begin{tabular}{rrrr}
    \toprule
    \multicolumn{4}{c}{$\rho=0.3$} \\
    \midrule
          & k-means & $\ell_1$-k-means & $\ell_0$-k-means \\
    %NW & 2000(0) & 226(47)  & 158(12)  \\
       PNW & 200(0) & 186(13)  & 157(12) \\
    PZW & 0(0) & 1760(35)  & 1799(1) \\
    \midrule
    \multicolumn{4}{c}{$\rho=0.1$} \\\midrule
          & k-means & $\ell_1$-k-means & $\ell_0$-k-means \\
  %  NW & 2000(0) & 234(46)  & 155(18)  \\
        PNW & 200(0) & 188(13)  & 150(18) \\
    PZW & 0(0) & 1754(33)  & 1799(1) \\
    \bottomrule
    \end{tabular}%
  \label{tab:weight}%
\end{table}%

\subsection{Evaluation on application to mouse Brain Atlas Data}
In this subsection, we further evaluate and compare the clustering and noise feature eliminating capability of the respective algorithms by applying to a Allen Developing Mouse Brain Atlas data~\cite{lein2006, ji2012}. This data set contains {\it in situ} hybridization gene expression pattern images of a developing mouse brain across 7 developmental ages. The mouse brain is imaged into 3D space with voxels in a regular grid. The expression energy at each voxel for some gene is recorded as a numerical value. Through such operation, 7 data matrices associated with each of 7 developmental ages are obtained. In these data matrices, rows correspond to brain voxels and columns correspond to genes. With the development of mouse brain, the rows of energy matrices increase because as the size of brain grows larger, more and more voxels are needed to stabilize the resolution. The basic statistics of the data are listed as in Table \ref{level3}, and Figure \ref{slices} shows the sample slices of 7 developmental mouse brains with respect to the gene {\it Neurog1}. In deed, each voxel is annotated to a brain region manually, which can be viewed as the cluster labels.

\begin{table}[!h]
\tabcolsep 0pt \caption{Statistics of mouse brain data at annotation level 3.}\label{level3}
\vspace*{-6pt}
\begin{center}
\def\temptablewidth{1\textwidth}
{\rule{\temptablewidth}{1pt}}
\begin{tabular*}{\temptablewidth}{@{\extracolsep{\fill}}cccccccc}
Ages & E11.5  & E13.5 & E15.5 & E18.5 & P4 & P14 & P28\\
\hline
       Number of genes & 1724 & 1724  &  1724 & 1724 & 1724 & 1724 & 1724   \\
       Number of voxels & 7122 & 13194 & 12148& 12045 & 21845 & 24180 & 28023\\
       Number of regions & 20 & 20 & 20 & 20 & 20 & 19 & 20\\

       \end{tabular*}
       {\rule{\temptablewidth}{1pt}}
       \end{center}
       \end{table}

\begin{figure}
\begin{center}
  \includegraphics[width=450pt]{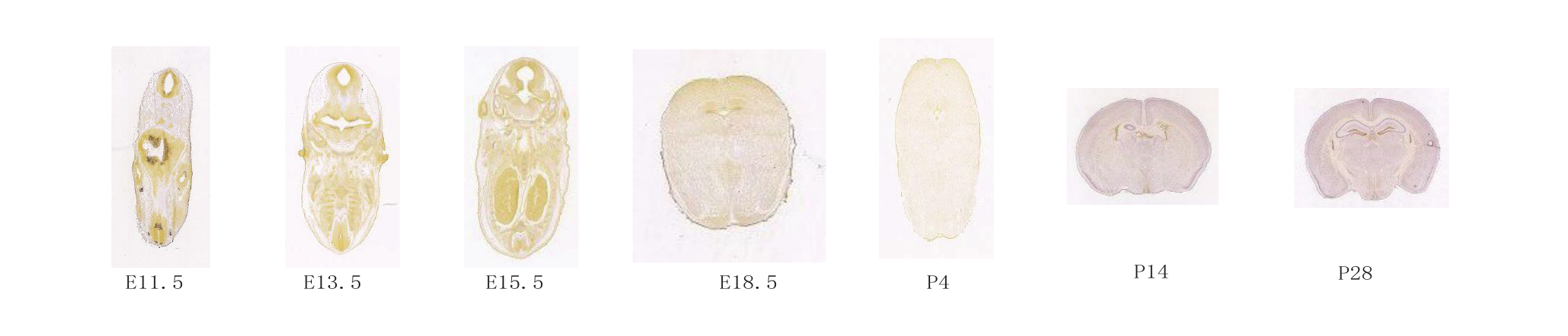}
\caption{\emph{Selected sample slices of 7 developmental mouse brains with respect to the gene Neurog1.}}\label{slices}
  \end{center}
\end{figure}

\begin{table}[!h]
\tabcolsep 0pt \caption{The CER values of clustering when the algorithms are applied to Allen Developing Mouse Brain Atlas data.}\label{CERlevel3}
\vspace*{-6pt}
\begin{center}
\def\temptablewidth{1\textwidth}
{\rule{\temptablewidth}{1pt}}
\begin{tabular*}{\temptablewidth}{@{\extracolsep{\fill}}cccccccc}
Ages & E11.5  & E13.5 & E15.5 & E18.5 & P4 & P14 & P28\\

\hline
 k-means  & 0.1610  &  0.1877  & {\bf 0.2055}  &   0.2369  &  0.3444  &  0.3628  &  0.3599\\
 $\ell_1$-k-means  & 0.1662  &  0.1985  &  0.2221  &  0.2425  &  0.3308  &  { 0.3593}  & \bf 0.3470\\
 $\ell_{0}$-k-means  & {\bf 0.1605}  & {\bf 0.1842} &  0.2259  &  {\bf 0.2358}  & {\bf 0.3306}  &  {\bf 0.3580}  &{ 0.3505}
       \end{tabular*}
       {\rule{\temptablewidth}{1pt}}
       \end{center}
\end{table}
\begin{table}[!h]
\tabcolsep 0pt \caption{The NW values of clustering
when the algorithms applied to Allen Developing Mouse Brain Atlas data.}\label{NWlevel3}
\vspace*{-6pt}
\begin{center}
\def\temptablewidth{1\textwidth}
{\rule{\temptablewidth}{1pt}}
\begin{tabular*}{\temptablewidth}{@{\extracolsep{\fill}}cccccccc}
Ages & E11.5  & E13.5 & E15.5 & E18.5 & P4 & P14 & P28\\

\hline
 k-means & 1723     &   1724 &       1724     &   1724 &       1720   &     1724&        1724\\
 $\ell_1$-k-means   &717      & 672    &  659     & 642   &  446   & 224  &  1724\\
 $\ell_{0}$-k-means  &100   & 660      &  100 &  1600      & 199  &  322    &  1068
        \end{tabular*}
       {\rule{\temptablewidth}{1pt}}
       \end{center}
\end{table}

We have applied the $\ell_0$-k-means, $\ell_1$-k-means, and standard k-means respectively to the 7 data sets (matrices). The application results (the CER values and the feature selection information) are shown in Table \ref{CERlevel3} and \ref{NWlevel3}. From Tables \ref{CERlevel3}, we can find that the $\ell_0$-k-means in most the cases outperforms the $\ell_1$-k-means and standard k-means, almost always with lower CER values.

The main improvement of the $\ell_0$-k-means is interpretability, since it often keeps relatively low CER values while using minimal number of features (nonzero weights $\bw$). The reason may be that the $\ell_0$-k-means can eliminate more noise features compared with others. Here we only focus on the the last two postnatal stages (P14 and P28),  because the differentiation of gene function is much more discriminative when a mouse is at the postnatal stage. Firstly, we investigate the P14 data. We observe that there is a few "noisy" genes which has been eliminated by the $\ell_0$-k-means and involved by the $\ell_1$-k-means. For instance, gene ‘Scn4b’, whose official name is ‘sodium channel, voltage-gated, type IV, beta subunit’, is highly related with the protein composition of sodium channel beta subunits~\cite{medeiros2007scn4b}. These subunits interact with voltage-gated alpha subunits to change sodium channel kinetics. The Gene Ontology (GO) annotations about this gene also include ion channel binding. In short, the protein encoded by this gene is one of the main elements of controlling the electrical signal transmission activity in cells including nerve, muscle, and neuroendocrine cell types. Therefore, it is much more reasonable to consider this gene as a noise feature since its function is uniformly supportive in the whole brain and its usage to distinguish different regions might not be effective~\cite{allenapi}. Thus, this supports that the proposed $\ell_0$-k-means detects the corrected noise gene in our experiment.
We also evaluate the performance of two methods respectively on P56 data. We observe that the weights computed by the $\ell_1$-k-means are all nonzero. In contrast, some features are still identified as noise features by proposed $\ell_0$-k-means with zero weights. This result is more consistent with the prior knowledge about genes listed in the database of Allen Institute~\cite{genecard}. Overall, the experiments demonstrate that $\ell_0$-k-means exhibits an outperforming capacity on the noise feature detection task.

\section{Conclusion}\label{conclusions}
In the article, we introduced two new yet rigorous concepts of optimal partition and noise features for high-dimensional data clustering problem. Motivated by these new concepts, we proposed a new sparse clustering framework with the $\ell_{\infty}/\ell_0$ penalty to eliminate noise features and yield the optimal partition of a data matrix simultaneously. As a realization of the framework, we suggested an $\ell_0$-k-means algorithm for comparing with the existed $\ell_1$-k-means algorithm, which used a very efficient closed form solution to solve the resultant non-convex and non-smooth optimization problem. The suggested $\ell_0$-k-means algorithm is theoretically analyzed and experimentally assessed. Based on the theoretical analysis and experiment studies, we can summarize the main contributions and significance of the present research as follows.

\begin{itemize}
\item The concepts of optimal partition and noise features are rigorously and quantitatively defined for high-dimensional data clustering problems. The new defined concepts cater for the analysis and application of clustering in the high-dimensional setting in which dimension may vary, that is, the number of features can grow as the number of sample size. We have shown (Theorem \ref{optimial_partition_theory}) that in usual cases the reasonableness of optimal partition and existence of noise features to support appropriateness of the new definitions for high-dimensional data clustering problems.

\item An efficient new sparse clustering algorithm for high-dimensional data clustering, the $\ell_0$-k-means, is suggested within the framework of classical k-means formulation plus the $\ell_{\infty}/\ell_0$ penalty. We found the closed-form solution of the resultant non-convex optimization problem (Theorem \ref{theorem1}) which makes the proposed $\ell_0$-k-means very efficient. It is because that despite the time consumed by the standard k-means, the proposed algorithm could be solved in $O(p\lfloor s \rfloor)$ time which is acceptable in most real applications. The experiment studies support the high efficiency.

\item With the new definitions of the optimal partition and noise features, it is shown that $\ell_0$-k-means possesses the feature selection consistency property. This distinguishes the $\ell_0$-k-means from the existing $\ell_1$-k-means for which it is still open whether or not it is feature selection consistent. This theoretical success for $\ell_0$-k-means supports the validity and appropriateness of the new framework suggested in this paper.

\item The experiment studies show that the $\ell_0$-k-means has its own set of advantages compared with the $\ell_1$-k-means, standard k-means and some other related well-known clustering algorithms in generating lower classification error rate and exhibiting more stronger ability to eliminating more noise features. It is demonstrated also that the results yielded by the $\ell_0$-k-means is more interpretable in practice.
\end{itemize}
There are many problems that deserve to further examine along the same line of the present study. One problem is, for instance, if it is possible to establish the feature selection consistency property for the $\ell_1$-k-means within the framework of the present paper. Another problem is to generalize the established theory in this paper to the case the feature may be strongly correlated (or any correlated data set). It is also expected to develop a more generic framework within which more spare clustering method, not only k-means, can be uniformly studied and compared. All these problems are under our current study.

\section{Appendix: Proofs}\label{appendix}
\subsection{Complement Lemmas}
We provide some useful lemmas that support our proofs in this section. In the first lemma we reformulate BCSS to facilitating our derivation. And the second lemma is a tail bound of sub-exponential random variables which can be applied to get Lemma \ref{tailchi}. To estimate the upper bound of the optimal value of (\ref{fepsilon}), we introduce Lemma \ref{ecrlimit}. %
%The last lemma is a limitation processing which can help the comparison of the conditions of the theorem in the article.
\begin{lemma}\label{ajreform}
Under the same setting we have described at subsection \ref{consistensection}, we can obtain $a_j$ which we denoted in (\ref{aj}) has the reformulation
\begin{eqnarray}
a_j=\sum_{k=1}^K(\frac{\sum_{i\in \tilde{\mathcal{C}}_k}x_{ij}}{\sqrt{n\tilde{\pi}_k}})^2-(\frac{\sum_{i=1}^nx_{ij}}{\sqrt{n}})^2,
\end{eqnarray}
where $n_k,k=1,2,\dots,K$ is the number of sample size in cluster $\tilde{\mathcal{C}}_k$ and $\tilde{\pi}_k\triangleq n_k/n$. Therefore,
$$BCSS(\tilde{\mathcal{C}})=\sum_{j=1}^{p}a_j=\sum_{j=1}^{p}\Big(\sum_{k=1}^K(\frac{\sum_{i\in \tilde{\mathcal{C}}_k}x_{ij}}{\sqrt{n\tilde{\pi}_k}})^2-(\frac{\sum_{i=1}^nx_{ij}}{\sqrt{n}})^2\Big).$$
\end{lemma}
\begin{proof}
Based on the definition of $a_j,j=1,2,\dots,p$, we have
\begin{eqnarray}
 a_j&=&\frac{1}{2n}\sum_{i_1,{i_2}}(x_{i_1j}-x_{i_2j})^2-\sum_{k=1}^K\frac{1}{2n_k}\sum_{i_1,i_2\in \tilde{\mathcal{C}}_k}(x_{i_1j}-x_{i_2j})^2\\\nonumber
&=&\sum_{i}x_{ij}^2-\frac{1}{n}(\sum_{i}x_{ij})^2-
\sum_{k=1}^{K}(\sum_{i\in \tc_k}x_{ij}^2-\frac{1}{n_k}(\sum_{i\in \tilde{\mathcal{C}}_k}x_{ij})^2)\\\nonumber
&=&-\frac{1}{n}(\sum_{i}x_{ij})^2+
\sum_{k=1}^{K}\frac{1}{n_k}(\sum_{i\in \tilde{\mathcal{C}}_k}x_{ij})^2\\\nonumber
&=&\sum_{k=1}^{K}(\frac{\sum_{i\in \tilde{\mathcal{C}}_k}x_{ij}}{\sqrt{n\tilde{\pi}_k}})^2-
(\frac{\sum_{i=1}^{n}x_{ij}}{\sqrt{n}})^2.
\end{eqnarray}
\end{proof}

To implicitly describe the next lemma, we first denote a random variable $Z$ with mean $\mu=\mathbb{E}(Z)$ is {\it sub-exponential} if there are non-negative parameters $(v,b)$ such that
$$\mathbb{E}[\exp\{\lambda(Z-\mu)\}]\leq \exp\{\frac{v^2\lambda^2}{2}\}\ \text{for all}\ |\lambda|<\frac{1}{b}.$$
\begin{lemma}\label{subtail}(Sub-exponential Tail Bound)
Suppose that $Z$ is sub-exponential with parameters $(v,b)$. Then
$$\mathbb{P}[Z\geq \mu+t]\leq \left\{\begin{array}{cc}
  \exp\{-\frac{t^2}{2v^2}\} & \text{if}\ \ 0\leq t\leq\frac{v^2}{b} \\
  \exp\{-\frac{t}{2b}\}& \text{if}\ \ t>\frac{v^2}{b}
\end{array}\right..$$
\end{lemma}
\begin{proof}
It can be found in~\cite{foss2011introduction}.
\end{proof}

Based on Lemma \ref{subtail}, we can directly get the next lemma.
\begin{lemma}\label{tailchi}
Suppose $Z=\sum_{i=1}^{d}(Y_i+b_i)^2$, where $Y_i\sim \mathcal{N}(0,1)$. Then $Z$ is a sub-exponential random variable with parameters $(2\sqrt{d+2\sum_{i=1}^{d}b_i^2},4)$. And we have the tail bound
$$\mathbb{P}(|Z-\mathbb{E}Z|>t)\leq 2\exp(-\frac{t^2}{16\mathbb{E}Z})\quad \forall t<\mathbb{E}Z.$$
\end{lemma}
\begin{proof}
Suppose $\lambda<1/2$, we have
$$\mathbb{E}[\exp\left(\lambda((Y_i+b_i)^2-1-b_i^2)\right)]=
\exp(\frac{2b_i^2\lambda^2}{1-2\lambda}-\lambda)(1-2\lambda)^{-\frac{1}{2}}.$$
It is easy to verify that $\exp(-\lambda)(1-2\lambda)^{-\frac{1}{2}}\leq\exp(2\lambda^2)$, for any $|\lambda|<\frac{1}{4}$. Therefore, we have $$\exp(\frac{2b_i^2\lambda^2}{1-2\lambda}-\lambda)(1-2\lambda)^{-\frac{1}{2}}
<\exp((4b_i^2+2)\lambda^2)\quad \forall |\lambda|<\frac{1}{4}.$$
Based on this calculation, we know it is sub-exponential with parameters $(2\sqrt{2b_i^2+1},4)$. Summing up $(Y_i+a_i)^2$ together, we know $Z$ is sub-exponential with parameters $(2\sqrt{2\sum_ib_i^2+d},4)$. Using Lemma \ref{subtail}, we can get the above tail bounds.
\end{proof}
\begin{lemma}\label{ecrlimit}
Let $f(\cdot)$ be the function defined in (\ref{fepsilon}). Then
$f$ is a decreasing function where $f(0)=\sum_{k}\pi_{k}\mu_{k}^2$ and $f(\epsilon)<f(0),\forall \epsilon>0$.
\end{lemma}

\begin{proof}
$f$ is a decreasing function because the smaller $\epsilon$ is, the larger the definition field will be, which results in higher objective function value. For any $\pi_{k,k'}$, if $1-\sum_{k'}\max_k\pi_{k,k'}>0$ then we have (by the convexity of function $x^2$) $$\sum_{k'}\tilde{\pi}_{k'}\left(\frac{\sum_{k}\tilde{\pi}_{k,k'}\mu_k}{\tilde{\pi}_{k'}}\right)^2
<\sum_{k}\pi_{k}\mu_k^2
$$
and if $1-\sum_{k'}\max_k\pi_{k,k'}=0$, we have $$\sum_{k'}\tilde{\pi}_{k'}\left(\frac{\sum_{k}\tilde{\pi}_{k,k'}\mu_k}{\tilde{\pi}_{k'}}\right)^2
=\sum_{k}\pi_{k}\mu_k^2.
$$
Thus $f(0)=\sum_{k}\pi_{k}\mu_k^2$ and $f(\epsilon)<f(0)$, for all $\epsilon>0$. This completes the proof.
\end{proof}

\subsection{Proof of Theorem \ref{optimial_partition_theory}}
\begin{proof}
Based on Lemma \ref{ajreform}, we can calculate the expectation of the $BCSS$ for the $j^{th}$ feature:
\begin{eqnarray}
\mathbb{E}a_j(\tc)&=&\mathbb{E}\sum_{k=1}^K(\frac{\sum_{i\in \tilde{\mathcal{C}}_k}x_{ij}}{\sqrt{n\tilde{\pi}_k}})^2-(\frac{\sum_{i=1}^nx_{ij}}{\sqrt{n}})^2\\
&=&n\sum_{k=1}^K\tilde{\pi}_{k}\tilde{\mu}_{k}^2-n(\sum_{k=1}^K\tilde{\pi}_k\tilde{\mu}_k)^2+K-1
\end{eqnarray}
where $\tilde{\pi}_k$ is the proportion of cluster $k$'s size and $\tilde{\mu}_k$ is the mean value for cluster $k$. $\tilde{\mu}_k=\sum_{k'}\pi_{k,k'}\mu_{k‘}$, where $\pi_{k,k'}$ is the proportion of samples both in cluster $\tc_k$ and $C^*_{k'}$ and $\mu_{k'}$ is the expectation of samples in cluster $C^*_k$.

(I) For $p^*<j\leq p$, we have $\mathbb{E}x_{ij}=0$. This shows $\tilde{\mu}_k=0$, Therefore we know they are noise features $\mathbb{E}a_j(\tc)=K-1,\forall \tc$. For other features $j\leq p^*$, consider $\mathbb{E}a_j(\mathcal{C}^*)=n\sum_{k=1}^K{\pi}_{k}{\mu}_{k}^2-n(\sum_{k=1}^K{\pi}_k{\mu}_k)^2+K-1$. Since $n\sum_{k=1}^K{\pi}_{k}{\mu}_{k}^2-n(\sum_{k=1}^K{\pi}_k{\mu}_k)^2>0$ holds because of the convexity of function $x^2$, we know $\mathbb{E}a_j(\mathcal{C}^*)>K-1$. This shows they are relevant features.

(II) This holds naturally using the proof above.

(III) Using the convexity of function $x^2$, it is easily proved that $\mathbb{E}a_j(\tc)\leq\mathbb{E}a_j(C^*)$ and the equality holds only when $\tc=C^*$. Therefore we know this proposition is valid.
\end{proof}

\subsection{Proof of Theorem \ref{theorem1}}
\begin{proof}
We take an omission for this relatively easy proof.
\end{proof}
\subsection{Proof of Theorem \ref{theorem2}}
\begin{proof}
Suppose $\tilde{\mathcal{C}}=(\tc_1,
\dots,\tc_K)$ is any partition we have known,
and the number of samples in both the $k‘^{th}$ cluster $\tc_{k’}$ and the $k^{th}$ cluster $C_{k}^*$ is $n\cdot \pi_{kk‘}$. $ECR({\tilde{\mathcal{C}}})=1-\sum_k\max_{k}(\pi_{kk’})$. Moreover, suppose $x_{ij}=v_{ij}+\mu_{ij}$ where $v_{ij}$ obeys the standard normal distribution.

Let us set $y=-\sum_jn(\frac{\sum_{i=1}^{n}x_{ij}}{n})^2$, then
based on Lemma \ref{ajreform}, $I_1\triangleq BCSS({\tilde{\mathcal{C}}})-y$ equals to
\begin{eqnarray}
 &&\sum_{j=1}^{p}
 \sum_{k'=1}^{K}(\frac{\sum_{i\in \tc_{k'}}v_{ij}}{\sqrt{n\tp_{k'}}}+
 \sqrt{n\tilde{\pi}_{k'}}\tilde{\mu}_{k'j})^2
\end{eqnarray}
where $\tilde{\mu}_{k'}\triangleq\frac{\sum_{k}\pi_{kk'}\mu_{kj}}{\tilde{\pi}_{k'}}$.
The same as $I_1$ we have $$I_2\triangleq BCSS(\mathcal{C^*})-y=
 \sum_{j=1}^{p}\sum_{k=1}^{K}(\frac{\sum_{i\in C^*_k}v_{ij}}{\sqrt{n\pi_k}}+\sqrt{n\pi_{k}}\mu_{kj})^2.$$
According to Lemma \ref{tailchi}, we have the tail bounds for $I_1$ and $I_2$:
$$\mathbb{P}(I_1-\mathbb{E}[I_1|{\mathcal{C}}^*]>t_1|{\mathcal{C}}^*)\leq\exp(-\frac{t_1^2}{16\mathbb{E}[I_1|{\mathcal{C}}^*]}) \quad\forall t_1\leq \mathbb{E}[I_1|{\mathcal{C}}^*],$$
and
$$\mathbb{P}(I_2-\mathbb{E}[I_2|\mathcal{C}^*]<-t_2|\mathcal{C}^*)\leq\exp(-\frac{t_2^2}{16\mathbb{E}[I_2|\mathcal{C}^*]}) \quad\forall t_2\leq \mathbb{E}[I_2|\mathcal{C}^*].$$
where $\mathbb{E}[I_1|{\mathcal{C}}^*]=pK+p^*n\sum_{k'}\tilde{\pi}_{k'}\tilde{\mu}_{k'}^2$ and
$\mathbb{E}[I_2|\mathcal{C}^*]=pK+p^*n\sum_{k}\pi_{k}\mu_{k}^2$.
Since for any $\tilde{\mathcal{C}}$, we have $\mathbb{E}[I_1|\mathcal{C}^*]\leq \mathbb{E}[I_2|\mathcal{C}^*]$(by Theorem \ref{optimial_partition_theory}), then for any $0<\epsilon<1$,
\begin{equation*}
\max\limits_{ECR(\mathcal{C})>\epsilon} \mathbb{P}(I_1-\mathbb{E}[I_1|\mathcal{C}^*]>t_1|{\mathcal{C}}^*)\leq
\exp(-\frac{t_1^2}{16\mathbb{E}[I_2|{\mathcal{C}}^*]}) \quad\forall t_1\leq \mathbb{E}[I_1|{\mathcal{C}}^*].
\end{equation*}
Moreover, assume that $$t_2=t_1=\frac{\mathbb{E}[I_2|{\mathcal{C}}^*]-\mathbb{E}[I_1|{\mathcal{C}}^*]}{2},$$
then for $\hat{\mathcal{C}}\in\arg\max_{\mathcal{C}}BCSS(\bX)$ we have
\begin{align}
&\mathbb{P}(ECR({\hat{\mathcal{C}}})\geq\epsilon|\mathcal{C}^*)<\mathbb{P}(\mathbb{E}[I_2|\mathcal{C}^*]-t_2>I_2|\mathcal{C}^*)\\
&+\mathbb{P}(\exists\ {\mathcal{C}}\ s.t.\ 1-\sum_k\max_{k'}(\pi_{k'k})\geq \epsilon\ \text{and } I_1>\mathbb{E}[I_1|\mathcal{C}^*]+t_1|{\mathcal{C}}^*)\\
&\leq \mathbb{P}(I_2-\mathbb{E}[I_2|\mathcal{C}^*]<-t_2|\mathcal{C}^*)+K^N\max\limits_{ECR(\mathcal{C})>\epsilon} \mathbb{P}(I_1-\mathbb{E}[I_1|\mathcal{C}^*]>t_1|{\mathcal{C}}^*)\\
&\leq 2\exp(\max\limits_{ECR(\mathcal{C})>\epsilon}\Big\{-\frac{t_1^2}{16\mathbb{E}[I_2|\mathcal{C}^*]}\Big\}+n\ln K).
\end{align}
The $\log$ value of the last term equals to
$$\max\limits_{ECR(\mathcal{C})>\epsilon}-\frac{(\mathbb{E}[I_2|\mathcal{C}^*]-\mathbb{E}[I_1|\mathcal{C}^*])^2}{64\mathbb{E}[I_2|\mathcal{C}^*]}+n\ln K=n\ln K-\min\limits_{ECR(\mathcal{C})>\epsilon}\frac{1}{64}\frac{p^*n(\sum_{k}\pi_{k}\mu_{k}^2-
\sum_{k'}\tilde{\pi}_{k'}\tilde{\mu}_{k'}^2)^2}{\frac{pK}{p^*n}+\sum_{k}\pi_{k}\mu_{k}^2}$$
\begin{equation}\label{temp2}
  \leq n\ln K-\min\limits_{ECR(\mathcal{C})>\epsilon}\frac{1}{64}\frac{p^*n(\sum_{k}\pi_{k}\mu_{k}^2-
\sum_{k'}\tilde{\pi}_{k'}\tilde{\mu}_{k'}^2)^2}{K+\sum_{k}\pi_{k}\mu_{k}^2}\ (\text{by the conditioan }p\leq p^*n).
\end{equation}
In order to estimate the last term of (\ref{temp2}), we should evaluate the quantity $\sum_{k'}\tilde{\pi}_{k'}\tilde{\mu}_{k'}^2$, thus we consider a following optimization problem:
\begin{align}\label{fepsilon}
&\max\sum_{k'}\tilde{\pi}_{k'}\tilde{\mu}_{k'}^2\\
&s.t.\left\{
\begin{array}{l}\nonumber
 1-\sum_{k'}\max_k\pi_{k,k'}\geq \epsilon\\
 \pi_{k,k'}\geq 0\\
 \sum_{k'} \pi_{k,k'}=\pi_k\\
 \sum_{k} \pi_{k,k'}=\tilde{\pi}_{k'}\\
 \forall k, k'=1,\dots,K
\end{array}
\right.
\end{align}
where $\tilde{\mu}_{k'}=\frac{\sum_{k}\pi_{k,k'}\mu_k}{\tilde{\pi}_{k'}}$. We suppose $f(\epsilon)$ is the maximum of the optimization (\ref{fepsilon}). Based on Lemma \ref{ecrlimit}, we know that $f$ is a decreasing function and $f(0)=\sum_{k}\pi_{k}\mu_{k}^2>f(\epsilon)\geq (\sum_k\pi_k\mu_k)^2$ for all $0<\epsilon<1-\max_k\pi_k$.
Thus, we have the estimation for (\ref{temp2}):
$$(\ref{temp2})\leq  n\ln K-\frac{1}{64}\frac{p^*n(f(0)-
f(\epsilon))^2}{f(0)+K}.$$
Define $F(\cdot)$ to be the inverse function of $128\frac{f(0)+K}{(f(0)-f(\epsilon))^2}\ln K$ (Obviously $F$ is a decreasing function), therefore when $\epsilon=F(p^*)$, we have
$$n\ln K-\frac{1}{64}\frac{p^*n(f(0)-
f(\epsilon))^2}{f(0)+K}<-n\ln K.$$
Because $f(\epsilon)\geq (\sum_k\pi_k\mu_k)^2$, $p^*$ has to be bigger than the constant
$$\kappa\triangleq 128\frac{\sum_k\pi_k\mu_k^2+K}{(\sum_k\pi_k\mu_k^2-(\sum_k\pi_k\mu_k)^2))^2}\ln K$$ to make $F(\cdot)$ well defined.

In summary, if $p^*\geq\kappa$ and $p\leq p^*n$, $\mathbb{P}(ECR(\hat{\mathcal{C}})>F(\epsilon)|\mathcal{C}^*)\leq 2K^{-n}$, which completes the proof.
\end{proof}

\subsection{Proof of Theorem \ref{theorem3}}
\begin{proof}
Suppose $x_{ij}=v_{ij}+\mu_{ij}$ where $v_{ij}$ obeys the standard normal distribution. Based on Lemma \ref{ajreform}, we have the formulation of any feature is (Non-standardization):
\begin{equation}\label{a_ja_j}
a_j=\sum_{k=1}^{K}(\frac{\sum_{i\in \tc_k}v_{ij}}{\sqrt{n\tp_k}}+\frac{\sum_{k'}\pi_{k',k}\mu_k'}{\tp_k}\sqrt{n}\mu)^2-
(\frac{\sum_{i=1}^{n}v_{ij}}{\sqrt{n}}+\frac{\sum_{i=1}^{n}\mu_{ij}}{\sqrt{n}})^2.
\end{equation}
In our next discussion, the last term of (\ref{a_ja_j}) will be omitted due to two reasons. First, the data matrix $\bX$ should be normalized in practice, thus the last term of (\ref{a_ja_j}) is zero for normalized data. Second, the last term of (\ref{a_ja_j}) is a constant regardless of the partition. What's more, discussing Theorem \ref{theorem3} only by the first $K$ terms of (\ref{a_ja_j}) shows the essential ideas for the proof. (we can send a proof with the discussion of the last term of (\ref{a_ja_j}) for interested readers).

For any partition ${\mathcal{C}}$ s.t. $ECR({\mathcal{C}})<F(p^*)$, $p^*$ is larger than $M$ (a constant) such that $K+\sum_{k}\pi_{k}\mu_{k}^2 n\geq\mathbb{E}[a_1|\mathcal{C}^*]\geq K+\frac{1}{2}\sum_{k}\pi_{k}\mu_{k}^2 n$. This can be achieved because when $\epsilon\rightarrow 0,$ $\tilde{\mu}_k\rightarrow \mu_k$ and $\tilde{\pi_k}\rightarrow \pi_k$ for any $k$.
Based on Lemma \ref{tailchi}, the probability that the smallest BCSS of the relevant features is smaller than $z=K+\frac{1}{4}\sum_{k}\pi_{k}\mu_{k}^2$ equals to
\begin{equation}\begin{split}
&\mathbb{P}\left\{\min_{j\leq p^*}a_j<K+\frac{1}{4}\sum_{k}\pi_{k}\mu_{k}^2n|\mathcal{C}^*\right\}\\&
=\mathbb{P}\left\{\exists j\leq p^*,a_j<K+\frac{1}{4}\sum_{k}\pi_{k}\mu_{k}^2n|\mathcal{C}^*\right\}\\&
\leq p^*\mathbb{P}\left\{a_1<K+\frac{1}{4}\sum_{k}\pi_{k}\mu_{k}^2n|\mathcal{C}^*\right\}\\&
\leq p^*\mathbb{P}\left\{a_1-\mathbb{E}a_1<-\frac{1}{4}\sum_{k}\pi_{k}\mu_{k}^2n|\mathcal{C}^*\right\}\\&
\leq p^*\exp\left(-n\mu^2\frac{(\sum_{k}\pi_{k}\mu_k^2)^2}{256\mathbb{E}a_1}\right)\\&
\leq p^*\exp\left(-n\mu^2\frac{(\sum_{k}\pi_{k}\mu_k^2)^2}{256(\frac{K}{n}+\sum_k\pi_{k}\mu_k^2)}\right)\\&
\leq p^*\exp\left(-n\frac{\sum_{k}\pi_{k}\mu_k^2}{257}\right)\ (\text{for large enough } n).
\end{split}
\end{equation}
Based on the conditions of the theorem, we have $P(\min\limits_{j\leq m}a_j>z|\mathcal{C}^*)\rightarrow 1$.

On the other hand, the probability that the biggest BCSS of noise features is bigger than $z$ equals to
\begin{equation}\begin{split}
&\mathbb{P}\left\{\max_{j> m}a_j>z|\mathcal{C}^*\right\}\\&
=\mathbb{P}\left\{\exists j> m,a_j>z|\mathcal{C}^*\right\}\\&
\leq (p-p^*)\mathbb{P}\left\{a_p>K+\frac{1}{4}\sum_{k}\pi_{k}\mu_{k}^2n\right\}\\&
 \leq (p-p^*)\exp\left(-\frac{(\sum_{k}\pi_{k}\mu_{k}^2n)^2}{256K}\right)\\&
\leq(p-p^*)\exp\left(-n\frac{\sum_{k}\pi_{k}\mu_k^2}{257}\right) (\text{ for large n}).\\&
\end{split}
\end{equation}
Now consider the partition $\hat{\mathcal{C}}\in\arg\max_{\mathcal{C}}BCSS(\bX)$. With the Stirling's approximation and Theorem \ref{theorem2}, we have the following estimates: %($\epsilon=F(p^*\mu^2)$)
\begin{eqnarray}\nonumber
  &&\mathbb{P}\left\{\text{for } \hat{\mathcal{C}}, \min_{j\leq p^*}a_j<\max_{j>p^*}a_j|\mathcal{C}^*\right\}\leq
  \mathbb{P}\left\{\exists {\mathcal{C}}\ s.t.\  ECR({\mathcal{C}})<\epsilon, \min_{j\leq p^*}a_j<\max_{j>p^*}a_j|\mathcal{C}^*\right\}+2K^{-n}\\
  &&\leq {n \choose \lfloor \epsilon n \rfloor}(K-1)^{\epsilon n}\left(\mathbb{P}\left\{\min_{j\leq p^*}a_j<y|\mathcal{C}^*\right\}+\mathbb{P}\left\{\max_{j>p^*}a_j>y|\mathcal{C}^*\right\}\right)+2K^{-n}\\
  &&\leq \exp\left\{H(\epsilon)n+\epsilon n\ln(K-1)+\ln(p)-n\frac{\sum_{k}\pi_{k}\mu_k^2}{257}\right\}+2K^{-n}\\
\end{eqnarray}
where $H(\epsilon)=-[\epsilon\log\epsilon+(1-\epsilon)\log(1-\epsilon)]$.
We can set $\epsilon$ so small yet positive such that
$$\exp\left\{H(\epsilon)n+\epsilon n\ln(K-1)+\ln(p)-n\frac{\sum_{k}\pi_{k}\mu_k^2}{257}\right\}<
\exp\left\{\ln(p)-n\frac{\sum_{k}\pi_{k}\mu_k^2}{258}\right\}.$$
Then $\mathbb{P}\left\{\text{for } \hat{\mathcal{C}}, \min_{j\leq p^*}a_j\geq \max_{j>p^*}a_j|\mathcal{C}^*\right\}\rightarrow 1$ if $p=o(\exp(n\frac{\sum_{k}\pi_{k}\mu_k^2}{258}))$ and $n\rightarrow\infty$. This completes the proof.
\end{proof}

%%%%%%%%%%%%%%%%%%%%%%
\paragraph{Acknowledgements.}  Xiangyu Chang gratefully acknowledges the support
of National Basic Research Program of China (973 Program): 2013CB329404 and  Key Project of NSF of China: 11131006.
We thanks Prof. Jiangshe Zhang (Xi'an Jiaotong University) and Prof. Shuiwang Ji (Old Dominion University) for helpful discussions.
\bibliographystyle{unsrt}
\bibliography{sparseclustering}

\end{document}